\documentclass[onefignum,onetabnum]{siamonline220329}

\usepackage{breqn}

\usepackage{amsfonts}
\usepackage{graphicx}
\usepackage{epstopdf}
\usepackage{algorithmic}
\usepackage{subfigure}
\usepackage{amsmath, amssymb}
\usepackage{multirow}
\allowdisplaybreaks
\usepackage{xr-hyper}
\usepackage{hyperref}
\hypersetup{colorlinks,breaklinks,
             linkcolor=blue,urlcolor=blue,
             anchorcolor=blue,citecolor=blue}
\usepackage{cleveref}
\usepackage{cancel}
\usepackage{bbm}
\usepackage{xcolor}
\usepackage{booktabs}
\usepackage{cite}

\newcommand{\mcl}{\mathcal}

\newcommand{\mbb}{\mathbb}
\newcommand{\lp}{\left(}
\newcommand{\rp}{\right)}

\newcommand{\R}{M} 

\renewcommand{\L}{\mathcal{L}}
\newcommand{\A}{\mathcal{A}}
\newcommand{\C}{\mathcal{D}}
\newcommand{\s}{\mathcal{S}}
\renewcommand{\epsilon}{\varepsilon}

\DeclareMathOperator*{\argmin}{arg\!min}
\DeclareMathOperator*{\argmax}{arg\!max}
\DeclareMathOperator*{\dist}{dist}
\let\div\relax
\DeclareMathOperator{\div}{div}

\ifpdf
  \DeclareGraphicsExtensions{.eps,.pdf,.png,.jpg}
\else
  \DeclareGraphicsExtensions{.eps}
\fi

\usepackage{enumitem}
\setlist[enumerate]{leftmargin=.5in}
\setlist[itemize]{leftmargin=.5in}

\newsiamremark{remark}{Remark}
\newsiamremark{hypothesis}{Hypothesis}
\crefname{hypothesis}{Hypothesis}{Hypotheses}
\newsiamthm{claim}{Claim}

\newsiamthm{policy}{Policy}
\newsiamthm{assumption}{Assumption}

\headers{PWLL Active Learning}{K. Miller and J. Calder}

\title{Poisson reweighted Laplacian uncertainty sampling for graph-based active learning\thanks{Submitted to the editors DATE.
{\bf Source code: } \url{https://github.com/millerk22/rwll_active_learning}
\funding{JC was supported by NSF grant DMS:1944925,
the Alfred P. Sloan foundation, and a McKnight Presidential Fellowship}}}

\author{Kevin Miller\thanks{The Oden Institute of Computational Engineering and Sciences, University of Texas, Austin, TX 
  (\email{ksmiller@utexas.edu}).}
\and Jeff Calder\thanks{Department of Mathematics, University of Minnesota, Twin Cities, MN 
  (\email{jwcalder@umn.edu}).}
}

\usepackage{amsopn}

\begin{document}

\maketitle

\begin{abstract}
We show that uncertainty sampling is sufficient 
to achieve exploration versus exploitation in graph-based active learning, as long as the measure of uncertainty properly aligns with the underlying model and the model properly reflects uncertainty in unexplored regions. In particular, we use a recently developed algorithm, Poisson ReWeighted Laplace Learning (PWLL) for the classifier and we introduce an acquisition function designed to measure uncertainty in this graph-based classifier that identifies unexplored regions of the data. We introduce a diagonal perturbation in PWLL which produces exponential localization of solutions, and controls the \emph{exploration} versus \emph{exploitation} tradeoff in active learning. We use the well-posed continuum limit of PWLL to rigorously analyze our method, and present experimental results on a number of graph-based image classification problems.
\end{abstract}
\begin{keywords}
active learning, uncertainty sampling, graph Laplacian, continuum limit, partial differential equations
\end{keywords}

\begin{MSCcodes}
35J15, 35J20, 68T05, 35Q68
\end{MSCcodes}

\tableofcontents
\section{Introduction} \label{sec:intro}

Supervised machine learning algorithms rely on the ability to acquire an abundance of labeled data, or data with known labels (i.e., classifications). While unlabeled data---data {\it without} known labels---is ubiquitous in most applications of interest, obtaining labels for such training data can be costly. 
Semi-supervised learning (SSL) methods leverage unlabeled data to achieve an accurate classification with significantly fewer training points. Simultaneously, the choice of training points can significantly affect classifier performance, especially due to the limited size of the training set of labeled data in the case of SSL.  
Active learning seeks to judiciously select a limited number of {\it query points} from the unlabeled data that will inform the machine learning task at hand.
These points are then labeled by an expert, or human in the loop, with the aim of significantly improving the performance of the classifier. 

While there are various paradigms for active learning \cite{settles_active_2012}, we focus on {\it pool-based} active learning wherein an unlabeled pool of data is available at each iteration of the active learning process from which query points may be selected. This paradigm is the natural fit for applying active learning in conjunction with semi-supervised learning since the unlabeled pool is also used by the the underlying semi-supervised learner. These query points are selected by optimizing an {\it acquisition function} over the discrete set of points available in the unlabeled pool of data. That is, if $\mcl U \subset \mcl X$ is the set of currently unlabeled points in a pool of data inputs $\mcl X \subset \mbb R^d$, then the active learning process at each iteration selects the next query point $x^\ast \in \mcl U$ to be the minimizer of a real-valued acquisition function
\[
    x^\ast = \argmin_{x \in \mcl U} \ \mcl A(x),
\]
where $\mcl A$ can depend on the current state of labeled information (i.e., the labeled data $\mcl L = \mcl X - \mcl U$ and corresponding labels for points in $\mcl L$). 

The above process (policy) for selecting query points is \textit{sequential} as only a single unlabeled point is chosen to be labeled at each iteration, as opposed to the \textit{batch} active learning paradigm. In batch active learning, a set of query points $\mcl Q \subset \mcl U$ is chosen at each iteration. While this is an important extension of the sequential paradigm and is an active area of current research in the literature \cite{sener_active_2018, gal_deep_2017, vahidian_coresets_2020, miller_model-change_2021}, we focus on the sequential case in this work.

Acquisition functions for active learning have been introduced for various machine learning models, especially support vector machines \cite{tong_support_2001, jiang_minimum-margin_2019, balcan_margin_2007}, deep neural networks\cite{gal_deep_2017, sener_active_2018, kushnir_diffusion-based_2020, shui_deep_2020, simeoni_rethinking_2021}, and graph-based classifiers \cite{zhu_combining_2003, ji_variance_2012, ma_sigma_2013, qiao_uncertainty_2019, miller_model-change_2021, murphy_unsupervised_2019}. We focus on graph-based classifiers for our underlying semi-supervised learning model due to their ability to capture clustering structure in data and their superior performance in the {\it low-label rate regime}---wherein the labeled data constitutes a very small fraction of the total amount of data. 
Most active learning methods for deep learning assume a moderate to large amount of initially labeled data to start the active learning process. While there is exciting progress in improving the low-label rate performance of deep semi-supervised learning \cite{sohn2020fixmatch, sellars2022Laplacenet, zheng2022Simmatch} and few-shot learning \cite{zhang2022differentiable, he2022attribute}, we restrict the focus of this paper to well-established graph-based paradigms for this setting. 

An important aspect in the application of active learning in real-world datasets is the inherent tradeoff between using active learning queries to either explore the given dataset or exploit the current classifier's inferred decision boundaries. This tradeoff is reminiscent of the similarly named ``exploration versus exploitation'' tradeoff in reinforcement learning. In active learning, it is important to thoroughly explore the dataset in the early stages, and exploit the classifier's information in later stages. Algorithms that exploit too quickly can fail to properly explore the dataset, potentially missing important information, while algorithms that fail to exploit the classifier in later stages can miss out on some of the main benefits of active learning.

In this work, we provide a simple, yet effective, acquisition function for use in graph-based active learning in the low-label rate regime that provides a natural transition between exploration and exploitation summarized in a single hyperparameter. We demonstrate through both numerical experiments and theoretical results that this acquisition function explores prior to exploitation. We prove theoretical guarantees on our method through analyzing the continuum limit partial differential eqauation (PDE) that serves as a proxy for the discrete, graph-based operator. This is a novel approach to providing sampling guarantees in graph-based active learning. We also provide experiments on a toy problem that illustrates our theoretical results, and the importance of the exploration versus exploitation hyperparameter in our method.

\subsection{Previous work} \label{subsec:prev-work}

The theoretical foundations in active learning have mainly focused on proving sample-efficiency results for linearly-separable datasets---oftentimes restricted to the unit sphere \cite{balcan2009agnostic, dasgupta_coarse_2006, hanneke_bound_2007}---for low-complexity function classes using disagreement or margin-based acquisition functions \cite{hanneke_theory_2014, hanneke_minimax_2015, balcan2009agnostic, balcan_margin_2007}. These provide convenient bounds on the number of active learning choices necessary for the associated classifier to achieve (near) perfect classification on these datasets with simple geometry. 
In contrast, much of the focus for theoretical work in graph-based active learning leverages assumptions on the clustering structure of the data that is assumed to be captured in the graph structure \cite{murphy_unsupervised_2019, dasarathy_s2_2015}, which sometimes is assumed to be hierarchical \cite{dasgupta_hierarchical_2008, dasgupta_two_2011, cloninger_cautious_2021}. A central priority in this line of inquiry establishes guarantees that, given assumptions about the clustering structure of the observed dataset $\mcl X$, the active learning method in question will query points from \textit{all} clusters (i.e., ensure exploration). The low-label rate regime of active learning---the focus of of this current work---is the natural setting for establishing such theoretical guarantees. 

Laplacian Learning \cite{zhu_semi-supervised_2003} has been a common graph-based semi-supervised learning model for a number of graph-based active learning methods \cite{zhu_combining_2003, ji_variance_2012, ma_sigma_2013, jun_graph-based_2016}. However, little work has been done to provide theoretical guarantees for these methods, possibly due to the inherent difficulty in proving meaningful estimates on the solutions of discrete graph equations. Other important works in active learning have focused primarily on improving the performance of deep neural networks via active learning with either (1) moderate to large amounts of labeled data available to the classifier \cite{gal_deep_2017, zhu_robust_2019} or (2) coreset methods that are agnostic to the observed labels of the labeled data seen throughout the active learning process \cite{sener_active_2018, vahidian_coresets_2020}. Our current work is focused on the \emph{low-label rate regime}, which is an arguably more fitting regime for semi-supervised and active learning. Furthermore, in contrast to coreset methods, our acquisition function directly depends on the observed classes of the labeled data.

Graph neural networks (GNN) \cite{welling2016semi, zhou2018graph} are an important area of graph-based methods for machine learning, and various methods for active learning have been proposed \cite{hu2020policy, ijcai2018p296, cai2017active, Zhang_Tong_Xia_Zhu_Chi_Ying_2022}. GNNs consider network graphs whose connectivity is a priori determined via metadata relevant to the task (e.g., co-authorship in citation networks) and then use the node-level features to learn representations and transformations of features for the learning task. In contrast, we consider similarity graphs where the connectivity structure is determined only by the node-level features and directly learn a node function on this graph structure.

Continuum limit analysis of graph-based methods has been an active area of research for providing rigorous analysis of graph-based learning \cite{calder_consistency_2019, calder2022improved, calder_poisson_2020, calder2020properly, calder2018game, slepcev2019analysis,dunlop2020large}. In this analysis, a discrete graph is viewed as a random geometric graph that is sampled from a density $\rho: \mbb R^d \rightarrow \mbb R_+$ defined in a high-dimensional space (possibly constrained to a manifold $\mcl M \subset \mbb R^d$ therein). The graph Laplacian matrix can be analyzed via its continuum-limit counterpart, which is a second order density weighted diffusion operator (or a weighted Laplace-Beltrami operator on the manifold).  An important development relevant to the current work is the Properly Weighted Graph Laplacian \cite{calder2020properly}, which reweights the graph in the Laplacian learning model of \cite{zhu_semi-supervised_2003} to correct for the degenerate behavior of Laplacian learning in the extremely low-label rate regime. This provides the setting for a well-defined, properly scaled graph-based semi-supervised learning model that we use in our current work to provide rigorous bounds on the acquisition function values to control the exploration versus exploitation tradeoff.

In order to apply active learning in practice, it is essential to design computationally efficient acquisition functions. 
Much of the current literature has sought to design more sophisticated methods that often have higher computational complexity (e.g., requiring the full inversion of the graph Laplacian matrix).
Uncertainty sampling \cite{settles_active_2012} is an example of a computationally efficient acquisition function since it only requires the output of the classifier on the unlabeled data. However, uncertainty sampling methods will often mainly select query points that concentrate along decision boundaries while ignoring large regions of the dataset that are distant from any labeled points. Phrased in the terminology of the exploration versus exploitation tradeoff in reinforcement learning, uncertainty sampling is often  overly ``exploitative'' and often achieves poor overall accuracy in empirical experiments \cite{ji_variance_2012}. 

In contrast, methods such as variance optimization (VOpt) \cite{ji_variance_2012},  $\Sigma$-Opt \cite{ma_sigma_2013}, Coresets \cite{sener_active_2018}, LAND \cite{murphy_unsupervised_2019}, and CAL \cite{cloninger_cautious_2021} could be characterized as primarily ``explorative'' methods. Oftentimes, however, such explorative methods, or other methods that are designed to both explore and exploit \cite{miller_model-change_2021, karzand_maximin_2020, zhu_combining_2003, gal_deep_2017} are more expensive to compute than simply using uncertainty sampling. For example, VOpt \cite{ji_variance_2012} and $\Sigma$-Opt \cite{ma_sigma_2013} require the computation and storage of a dense $N \times N$ covariance matrix that must be updated each active learning iteration. The work of \cite{miller_model-change_2021} proposed a computationally efficient adaptation of these methods via a projection onto a subset of the graph Laplacian's eigenvalues and eigenvectors.As a consequence of sometimes significantly poor performance from this spectral truncation method in our experiments, we provide a ``full'' computation of V-Opt and $\Sigma$-Opt in certain experiments by restricting the computation to only a subset of unlabeled data which allows us to bypass the need to invert the graph Laplacian matrix (Section \ref{sec:larger-datasets}). This heuristic, however, is still very expensive to compute at each active learning iteration making it not a viable option for moderate to large datasets in practice.

In this work, we show that uncertainty sampling, \textit{when properly designed for the graph-based semi-supervised learning model} can both explore and exploit, and outperforms existing methods in terms of computational complexity, overall accuracy, and exploration rates.

\subsection{Overview of paper} \label{sec:overview-contents}

The rest of the paper continues as follows. We begin in Section \ref{sec:model-setup} with a description of the Properly Weighted Laplacian Learning model from \cite{calder2020properly} that will be the underlying graph-based semi-supervised learning model for our proposed active learning method. We also introduce the minimum norm acquisition function in this section, along with other useful preliminaries for the rest of the paper. In Section \ref{sec:results}, we begin with illustrative experiments in two-dimensions to illustrate the delicate balance between exploration and exploitation in graph-based active learning. 
Section \ref{sec:larger-datasets} compares our proposed active learning method to other acquisition functions on larger, more ``real-world'' datasets that have been adapted to provide an experimental setup wherein exploration is essential for success in the active learning task. Thereafter, we present theoretical guarantees for the minimum norm acquisition function in the continuum limit setting in Section \ref{sec:theory}, along with an extended look at the theory in one dimension in Section \ref{sec:1d-theory}.

\subsection{Notation} \label{subsec:notation}

Let $\|\cdot\|_2$ denote the standard Euclidean norm where the space is inferred from the input. We let $|\cdot|$ denote either the absolute value of a scalar in $\mbb R$ or the cardinality of a set, where from context the intended usage should be clear. We denote the set of points $x \in \mcl X$ with  $x \not\in \mcl U$ as $\mcl X \setminus \mcl U$.

\section{Model setup and acquisition function introduction} \label{sec:model-setup}

Let $\mcl X = \{x_1, x_2, \ldots, x_N\} \subset \mbb R^d$ be a set of inputs for which we assume each $x \in \mcl X$ belongs to one of $C$ classes. Suppose that we have access to a subset $\mcl L \subset \mcl X$ of labeled inputs (\textit{labeled data}) for which we have observed the ground-truth classification $y(x) \in \{1, \ldots, C\}$ for each $x \in \mcl L$.
The rest of the inputs, $\mcl U := \mcl X \setminus \mcl L$, are termed the \textit{unlabeled data} as no explicit observation of the underlying classification have been seen for $x \in \mcl U$. The semi-supervised learning task is to use both $\mcl L$ and $\mcl U$, with the associated labels $\{y(x)\}_{x \in \mcl L}$, to infer the classification of the points in $\mcl U$. 

Sequential active learning extends semi-supervised learning by selecting a sequence of \textit{query points} $x_1^\ast, x_2^\ast, \ldots$ as part of an iterative process that alternates between (1) calculating the  semi-supervised classifier given the current labeled information and (2) selecting and subsequently labeling an unlabeled query point $x_n^\ast \in \mcl U_n$, where $\mcl U_n = \mcl X - \mcl L_n = \mcl X - (\mcl L \cup \{x_1^\ast, x_2^\ast, \ldots, x_{n-1}^\ast\})$. Labeling a query point $x_i^\ast$ consists of obtaining the corresponding label $y(x_n^\ast)$ and then adding $x_n^\ast$ to the set of labeled data from the current iteration, $\mcl L_{n} = \mcl L_{n-1} \cup \{x_n^\ast\}$. To avoid this cumbersome notation, however, we will drop the explicit dependence of $\mcl U_i, \mcl L_i$ on the iteration $i$ and simply refer to the unlabeled and labeled data at the \textit{current} iteration as respectively $\mcl U$ and $\mcl L$. 

Returning to the underlying semi-supervised learning problem, graph Laplacians have often been used to propagate labeled information from $\mcl L$ to $\mcl U$ \cite{zhu_semi-supervised_2003, bertozzi_diffuse_2016, calder_poisson_2020, calder2020properly, shi2017weighted, bertozzi2019graph, calder2018game,welling2016semi}. From the set of feature vectors $\mcl X$, consider a similarity graph $G(\mcl X, W)$ with weight matrix $w_{ij} = \kappa(x_i,x_j)$ that captures the similarity between inputs $x_i,x_j$ for each pair of points in $\mcl X$. We use $\mcl X$ to denote both the set of feature vectors as well as the node set for the graph $G$ to avoid introducing more notation. Laplacian learning \cite{zhu_semi-supervised_2003} is an important graph-based semi-supervised learning model for both this current work and many previous graph-based active learning works, and solves the constrained problem of identifying a graph function $u :\mcl X \rightarrow \mbb R^C$ via the minimization of 
\begin{align}\label{eq:lap-learning}
    \min_{u: \mcl X \rightarrow \mbb R^d}\ &\sum_{x_i,x_j \in \mcl X} w_{ij} \|u(x_i) - u(x_j)\|_2^2 \\
    \text{subject to }& u(x) = e_{y(x)} \text{ for } x \in \mcl L. \nonumber
\end{align}
The vector $e_{y(x)} \in \mbb R^C$ is the standard Euclidean basis vector in $\mbb R^C$ whose entries are all $0$ except the entry corresponding to the label $y(x)\in \{1, \ldots, C\}$.
The learned function $u$ that minimizes \eqref{eq:lap-learning} constitutes a harmonic extension of the given labels in $\mcl L$ to the unlabeled data since $u$ is a harmonic function on the graph. For the classification task, the inferred classification of $x \in \mcl U$ is then obtained by thresholding on the learned function's output at $x$, $u(x) \in \mbb R^C$. That is, the inferred classification $\hat{y}(x)$ for $x \in \mcl U$ is given by
\[
    \hat{y}(x) = \argmax_{c=1,2, \ldots, C} \ u_c(x),
\]
where $u_c(x)$ denotes the $c^{th}$ entry of $u(x)$.

Various previous works \cite{calder_rates_2020,calder2020properly, shi2017weighted, nadler2009infiniteunlabelled,flores2022analysis,calder_poisson_2020} have shown that when the amount of labeled information is small compared to the size of the graph (i.e., the \textit{low-label rate regime}), the performance of minimizers of \eqref{eq:lap-learning} degrades substantially. The solution $u$ becomes roughly constant with sharp spikes near the labeled set, and the classification tends to predict the same label for most datapoints. Of particular interest to the current work is the Properly Weighted Laplacian learning work in \cite{calder2020properly}, wherein a weighting $\gamma: \mcl X \rightarrow \mbb R_+$ that scales like $\operatorname{dist}(x, \mcl L)^{-\alpha}$ for $\alpha > d-2$ is used to reweight the edges in the graph to correct the singular behavior of solutions to \eqref{eq:lap-learning}. We use an improvement to the Properly Weighted Laplacian that is called Poisson ReWeighted Laplace Learning (PWLL) and will be described in detail in another paper \cite{calder2022poisson}. PWLL performs semi-supervised learning by solving the problem
\begin{align}\label{eq:rw-lap-learning}
    \min_{u: \mcl X \rightarrow \mbb R^d}\ &\sum_{x_i,x_j \in \mcl X} \gamma(x_i) \gamma(x_j)w_{ij} \|u(x_i) - u(x_j)\|_2^2 \\
    \text{subject to }& u(x) = e_{y(x)} \text{ for } x \in \mcl L, \nonumber
\end{align}
where the reweighting function $\gamma$ is computed by solving the graph Poisson equation 
\begin{equation}\label{eq:gamma_eq_discrete}
\sum_{x_j \in \mcl X} w_{ij}(\gamma(x_i) - \gamma(x_j)) = \sum_{x_k \in \mcl L}\left( \delta_{ik} - \tfrac{1}{N}\right) \ \ \ \text{for all } x_i\in \mcl X.
\end{equation}
In the previous work on the Properly Weighted Laplacian \cite{calder2020properly}, the weight $\gamma$ was explicitly chosen to satisfy $\gamma(x)\sim\operatorname{dist}(x, \mcl L)^{-\alpha}$, while in the PWLL, $\gamma$ is learned from the data, making the method more adaptive with fewer hyperparameters. The motivation for the Poisson equation \eqref{eq:gamma_eq_discrete} is that the continuum version of this equation is related to the fundamental solution of Laplace's equation, which produces the correct scaling in $\gamma$ near the labeled set. 

The reason for using PWLL is that minimizers of \eqref{eq:rw-lap-learning} have a well-defined continuum limit in the case when the amount of labeled data is fixed and the number of nodes $|\mcl X| = N \rightarrow \infty$. This will allow us to analyze the behavior of our proposed minimum norm acquisition function applied to the PWLL model in the continuum limit setting in Section \ref{sec:theory}.

\subsection{Solution decay parameter} \label{subsec:tau-decay}

We introduce an adaptation of \eqref{eq:rw-lap-learning} that increases the decay rate of the corresponding solutions away from labeled points via a type of Tikhonoff regularization in the variational problem. Controlling this decay will prove to be crucial for ensuring that query points selected via our minimum norm acquisition function (Section \ref{subsec:min-norm-af}) will explore the extent of the dataset prior to exploiting current classifier decision boundaries. Given $\tau \ge 0$, we consider solutions of the following variational problem
\begin{align}\label{eq:rw-lap-learning-tau}
    \min_{u: \mcl X \rightarrow \mbb R^d}\ &\sum_{x_i,x_j \in \mcl X} \gamma(x_i) \gamma(x_j)w_{ij} \|u(x_i) - u(x_j)\|_2^2 \ \  +\ \  \tau \sum_{x_i \in \mcl U} \|u(x_i)\|_2^2\\
    \text{subject to }& u(x) = e_{y(x)} \text{ for } x \in \mcl L. \nonumber
\end{align}
It is straightforward to see that for $\tau > 0$ the additional term in \eqref{eq:rw-lap-learning-tau} encourages the solution $u$ to have \textit{smaller} values away from the labeled data, where the values are fixed. When $\tau = 0$, we recover \eqref{eq:rw-lap-learning}. We will refer to this graph-based semi-supervised learning model as Poisson ReWeighted Laplace Learning with $\tau$-Regularization (PWLL-$\tau$). 

To illustrate the role of the decay parameter, let us consider a simple one dimensional version of this problem in the continuum of the form
\[\min_{u} \int_{a}^b u'(x)^2 + \tau u(x)^2\, dx,\]
where $[a,b]$ is the domain and the minimization would be restricted by some boundary conditions on $u$ (i.e., on the labeled set).  Minimizers of this problem satisfy the ordinary differential equation (i.e., the Euler-Lagrange equation) $\tau u - u'' = 0$, which has two linearly independent solutions $e^{\pm  \sqrt{\tau}x}$. Since the solution we are interested in is bounded, the exponentially growing one can be discarded, and we are left with exponential decay in the solutions with rate $ \sqrt{\tau}$ away from the labeled set. Thus, at least in this simple example, we can see how the introduction of the diagonal perturbation $\tau$ in PWLL leads to exponential decay of solutions, which is essential for the method to properly \emph{explore} the dataset. We postpone developing this theory further until Section \ref{sec:theory}.

\subsection{Minimum norm acquisition function} \label{subsec:min-norm-af}

We now introduce the acquisition function that we propose to properly balance exploration and exploitation in graph-based active learning in the PWLL-$\tau$ model. We simply use the Euclidean norm of the output vector at each unlabeled point, $x \in \mcl U$:
\begin{equation}\label{eq:min-norm-af}
    \mcl A(x) = \|u(x)\|_2 = \sqrt{u_1^2(x) + u_2^2(x) + \ldots + u_C^2(x)}.
\end{equation}
Due to the solution decay resulting from the $\tau$-regularization term in \eqref{eq:rw-lap-learning-tau}, unlabeled points that are far from all labeled points will have small Euclidean norm for their corresponding output vector. In the low-label rate regime, this property encourages query points selected by \eqref{eq:min-norm-af} to be spread out over the extent of the dataset, until a sufficient number of points have been labeled to ``cover'' the dataset. After this has been achieved in the active learning process, the learned functions for \eqref{eq:rw-lap-learning-tau} will have smaller norms in regions between labeled points of differing classes due to the rapid decay in solutions near the transition between classes. This described behavior reflects the desired properties for balancing exploration prior to exploitation in active learning. Through both numerical experiments and theoretical results, we demonstrate this acquisition function's utility for this purpose. 


The acquisition function \eqref{eq:min-norm-af} is a novel type of uncertainty sampling \cite{settles_active_2012}, wherein only the values of the learned function $u$ at each active learning iteration are used to determine the selection of query points. Indeed, one may interpret the small Euclidean norm of the learned function at an unlabeled node, $\|u(x)\|_2$, to reflect uncertainty about the resulting inferred classification, $\hat{y}(x)$. Other uncertainty sampling methods, such as \textit{smallest margin sampling} \cite{settles_active_2012}, also compute the uncertainty of the learned model at an unlabeled point via properties of the output vector $u(x) \in \mbb R^C$. However, these criterion often either (1) only compare 2 entries of the vector to compute a measure of margin uncertainty or (2) normalize the output vector to lie on the simplex to be interpreted as class probabilities. In both cases, these measures of uncertainty in the classification of unlabeled points in unexplored regions of the dataset might not be as emphasized by the acquisition function compared to points that lie near the decision boundaries of the learned classifier. In other words, most previous uncertainty sampling methods can often be characterized as solely exploitative and lack explorative behavior, with the results being decreased overall performance of the classifier on the dataset. Our minimum norm acquisition function \eqref{eq:min-norm-af}, however, is designed to prioritize the selection of query points in unexplored regions of the dataset which is properly reflected in the decay of the learned functions in the PWLL-$\tau$ model \eqref{eq:rw-lap-learning-tau}. In this sense, we are able to ensure exploration prior to exploitation in the active learning process using the minimum norm acquisition function \eqref{eq:min-norm-af} in the PWLL-$\tau$ model. 

\begin{remark}[Decay Schedule for $\tau$]
As we demonstrate through some toy experiments in Section \ref{sec:toy-experiments}, there is benefit to decreasing the value of $\tau \ge 0$ as the active learning process progresses in order to more effectively transition from explorative to exploitative queries. While there are various ways to design this, we simply identify a constant $\mu \in (0,1)$ so that the decreasing sequence of hyperparameter values $\tau_{n+1} = \mu \tau_n$ with initial value $\tau_0 > 0$ satisfies that $\tau_{2K} \le \epsilon$, where $\epsilon$ is chosen to be $\epsilon = 10^{-9}$. For our experiments, we set $K$ to be the number of clusters, which in the case of our tests is known a priori. In practice, this choice of $K$ would be a user-defined choice to control the ``aggressiveness'' of the decay schedule of $\tau$. For $n \ge 2K$, we set $\tau_n = 0$. Thus, we calculate
\[
    \mu = \lp \frac{\epsilon }{\tau_0}\rp^{\frac{1}{2K}} \in (0,1)
\]
which ensures a decaying sequence of $\tau$ values as desired. We note that an interesting line of inquiry for future research would be to investigate a more rigorous understanding of how to adaptively select $\tau \ge 0$ during the active learning process. We leave this question for future research and simply use the proposed decay schedule above.
\end{remark}



 In Table \ref{table:unc-sampling}, we introduce the abbreviations for and other useful information pertaining to the uncertainty sampling acquisition functions that we will consider in the current work---smallest margin, minimum norm, and minimum norm with $\tau$-decay uncertainty sampling.

\newcommand{\ra}[1]{\renewcommand{\arraystretch}{#1}}
\begin{table*}[h!]\centering \label{table:unc-sampling}
\ra{1.3}
\begin{scriptsize}
\begin{tabular}{@{}lccc@{}}\toprule
\textbf{Full Name} & \textbf{Abbreviation} & $\mcl A(x)$ & \textbf{Underlying Classifier} \\
\midrule 
Smallest Margin Unc.~Sampling & Unc.~(SM) & $u_{c_1^\ast}(x) - u_{c_2^\ast}(x)$ & PWLL \\
Minimum Norm Unc.~Sampling & Unc.~(Norm) & $\|u(x)\|_2$ & PWLL-$\tau$, fixed $\tau > 0$ \\
\multirow{2}{15em}{Minimum Norm Unc.~Sampling with $\tau$-decay} & Unc.~(Norm, $\tau \rightarrow 0$) & $\|u(x)\|_2$ & PWLL-$\tau$, decay $\tau \rightarrow 0$ \\
& & & \\
\bottomrule
\end{tabular}
\end{scriptsize}
\caption{Description of uncertainty sampling acquisition functions that will be compared throughout the experiments in the following sections. Unc.~(SM) considers the difference between the largest and second largest entries of the output vector $u(x)$, denoted by $c_1^\ast$ and $c_2^\ast$ respectively.}
\end{table*}

\section{Results}\label{sec:results}

In this section, we present numerical examples to demonstrate our claim that our proposed Unc.~(Norm) and Unc.~(Norm $\tau \rightarrow 0$) acquisition functions in the PWLL-$\tau$ model \eqref{eq:rw-lap-learning-tau} are effective at both exploration and exploitation. We begin in Section \ref{sec:toy-experiments} with a set of toy examples in 2-dimensions to facilitate visualizing the choices of query points during the active learning process and highlight the efficacy of implementing the $\tau$-decay in Unc.~(Norm, $\tau \rightarrow 0$) for balancing exploration and exploitation. In Section \ref{sec:isolet-results}, we recreate an experiment from \cite{ji_variance_2012} on the Isolet dataset \cite{uci} that highlighted the superior performance of the VOpt criterion compared to Unc.~(SM). We demonstrate that our proposed uncertainty sampling method Unc.~(Norm, $\tau \rightarrow 0$) achieves results comparable to VOpt on this dataset, essentially correcting the behavior of uncertainty sampling in this empirical study. 

In Section \ref{sec:larger-datasets}, we perform active learning experiments on larger, more ``real-world'' datasets. We use the \textbf{MNIST} \cite{lecun-mnisthandwrittendigit-2010}, \textbf{FASHIONMNIST} \cite{xiao2017fashionmnist}, and \textbf{EMNIST} \cite{cohen2017emnist} datasets, and we interpret the original ground-truth classes (e.g. digits 0-9 in \textbf{MNIST}) as \textit{clusters} on which we impose a different classification structure by grouping many clusters into a single class. This creates an experimental setting that necessitates exploration of initially unlabeled ``clusters'' in order to achieve high overall accuracy. We include similar experiments in Section \ref{smsec:imbalanced-results} of the Supplemental Material to verify the performance of the proposed method in the presence of disparate class and cluster sizes.

While most previous work in the active learning literature (both graph-based and neural network classifiers) demonstrate acquisition function performance with only accuracy plots, we suggest another useful quantity for comparing performances. In the larger experiments of Sections \ref{sec:larger-datasets} and \ref{smsec:imbalanced-results}, we plot \textit{the proportion of clusters that have been queried} as a function of active learning iteration. These plots reflect how efficiently an acquisition function explores the clustering structure of the dataset, as captured by how quickly the proportionality curve increases toward 1.0. These cluster exploration plots are especially insightful for assessing performance in low label-rate active learning. An acquisition function that properly and consistently explores the clustering structure of the dataset will achieve an average cluster proportion of 1.0 faster than other acquisition functions and within a reasonable number of active learning queries.

\subsection{Comment regarding comparison to other methods}

We comment here on a few notable active learning methods that are left out of our numerical comparisons. The LAND (Learning by Active Non-linear Diffusion) \cite{murphy_unsupervised_2019} and CAL (Cautious Active Learning) \cite{cloninger_cautious_2021} methods are important works in geometry-inspired active learning. In the LAND algorithm, Murphy and Maggioni use diffusion distances from a random walk interpretation of a similarity graph to select diverse sets of query points that are located in dense regions of the graph. Adjusting a model hyperparameter in the diffusion distances can reveal hierarchical clustering structure in the dataset which can encourage query points to be chosen at different resolution levels of the clustering structure. 

In a similar vein, the CAL algorithm by Cloninger and Mhaskar \cite{cloninger_cautious_2021} uses hierarchical clustering structure in the dataset to guide the query set selection process. By constructing a highly localized similarity kernel via Hermite polynomials, query points are selected at various resolution levels. Both the LAND and CAL algorithms have been shown to be effective at selecting query points in pixel classification for hyperspectral imagery applications. We, however, found that the current implementations of these algorithms were unable to scale to our larger experiments. 

Furthermore, we suggest that these methods may be more appropriately identified as ``coreset'' selection methods. Such methods leverage the geometry of the underlying dataset (e.g., the diffusion distances as captured by the similarity graph in LAND), but not the set of labels observed at labeled points during the active learning process. This is similar to other coreset methods that have been presented in both coreset and data summarization literature \cite{sener_active_2018, vahidian_coresets_2020, mirzasoleiman2017big}. In contrast, our uncertainty-based criterion in this work combines both geometric information about the data as captured by the similarity graph structure and the observed labels at each labeled point via the output classification at each iteration. This makes our method more similar to the primary flavor of active learning methods. For these two reasons, we omit direct numerical comparison with these other methods.

\subsection{Toy examples} \label{sec:toy-experiments}

We first illustrate our claim regarding our minimum norm uncertainty sampling criterion for graph-based active learning with synthetic datasets that are directly visualizable (i.e., the data lies in only two dimensions). The first experiment---which we refer to as the \textbf{Blobs} experiment---illustrates how a non-zero value for $\tau$ in the initial phase of active learning is crucial for ensuring exploration of the dataset. The second experiment---which we refer to as the \textbf{Box} experiment---illustrates the need to decrease the value of $\tau$ to ensure the transition from exploration to exploitation. These experiments also allow us to directly observe the qualitative characteristics of the active learning query choices in uncertainty sampling. 

\subsubsection{Blobs experiment} \label{sec:blobs-experiment}

\begin{figure}
    \centering
    \subfigure[Ground Truth]{\includegraphics[clip=True,trim=120 60 120 60,width=0.4\textwidth]{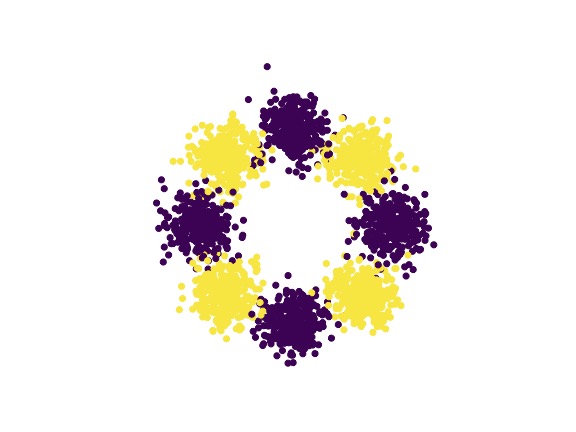}} 
    \subfigure[Accuracy Results]{\includegraphics[width=0.5\textwidth]{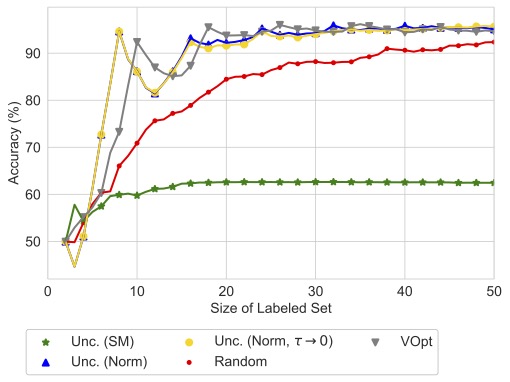}}
    \caption{Ground Truth (a) and Accuracy Results (b) for \textbf{Blobs} experiment. Notice that Unc.~(SM) achieves very poor overall accuracy. We show in Figure \ref{fig:unc-blobs-combined} that this is due to premature exploitation.}
    \label{fig:gt-blobs}
\end{figure}

The \textbf{Blobs} dataset is comprised of eight Gaussian clusters, each of equivalent size (300) and variance ($\sigma^2 = 0.17^2$), whose centers (i.e., means) lie evenly spaced apart on the unit circle. That is, each cluster $\Omega_i$ is defined by randomly sampling 300 points from a Gaussian with mean $\mu_i = (\cos(\pi i /4), \sin(\pi i/4))^T \in \mbb R^2$ and standard deviation $\sigma_i= \sigma = 0.17$. The classification structure of the clusters is then assigned in an alternating fashion, as shown in Figure \ref{fig:gt-blobs}(a). In each individual run of the experiment, one initially labeled point per \textit{class} combine to be the starting labeled set, and then 100 active learning query points are selected sequentially via a specified acquisition function. Different acquisition functions then define different runs of the experiment.

For each acquisition function, we ran 10 experiments with different initially labeled points. The average accuracy at each iteration of an experiment is plotted is Figure \ref{fig:gt-blobs}(b). The main purpose of this experiment is to compare and contrast the characteristics of the query points selected by Unc.~(SM), Unc.~(Norm), and Unc.~(Norm, $\tau \rightarrow 0$). For comparison and reference in these toy experiments, we include the results of using the VOpt\cite{ji_variance_2012} acquisition function as well as Random sampling (i.e., select $x_i^\ast \in \mcl U$ with uniform probability over $\mcl U$ at each iteration).

\begin{figure}
    \centering
    \subfigure[Unc.~(SM), Initial]{\includegraphics[clip=True,trim=120 60 100 60,width=0.32\textwidth]{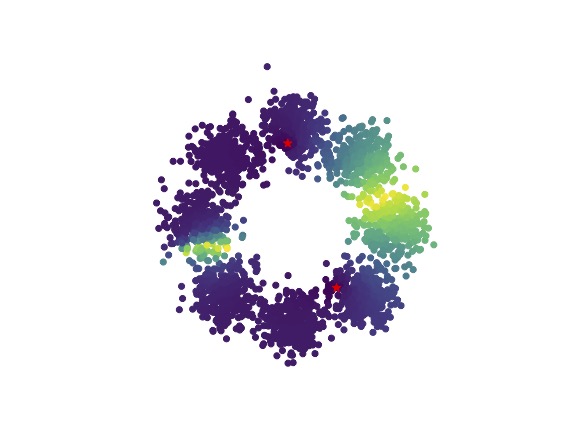}} 
    \subfigure[Unc.~(SM), Iter 9]{\includegraphics[clip=True,trim=120 60 100 60,width=0.32\textwidth]{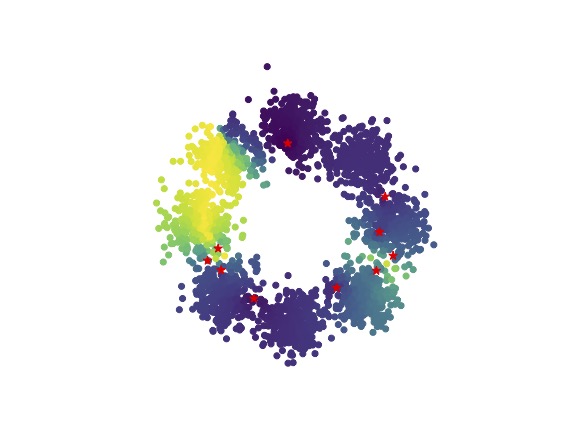}}
    \subfigure[Unc.~(SM), Iter 100]{\includegraphics[clip=True,trim=120 60 100 60,width=0.32\textwidth]{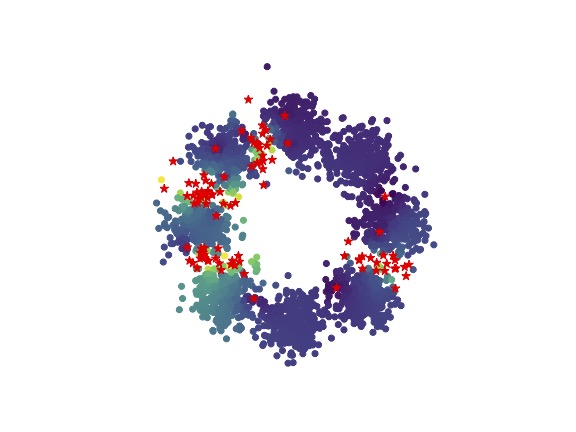}} \\
    \subfigure[Unc.~(Norm), Initial]{\includegraphics[clip=True,trim=120 60 100 60,width=0.32\textwidth]{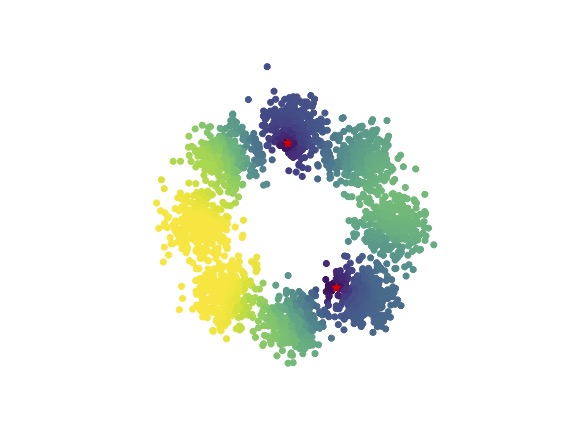}} 
    \subfigure[Unc.~(Norm), Iter 9]{\includegraphics[clip=True,trim=120 60 100 60,width=0.32\textwidth]{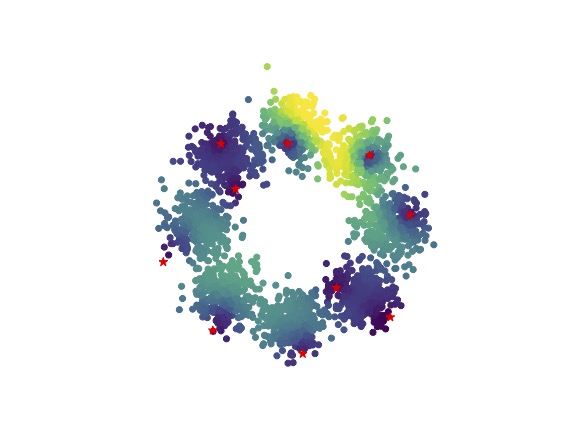}}
    \subfigure[Unc.~(Norm), Iter 100]{\includegraphics[clip=True,trim=120 60 100 60,width=0.32\textwidth]{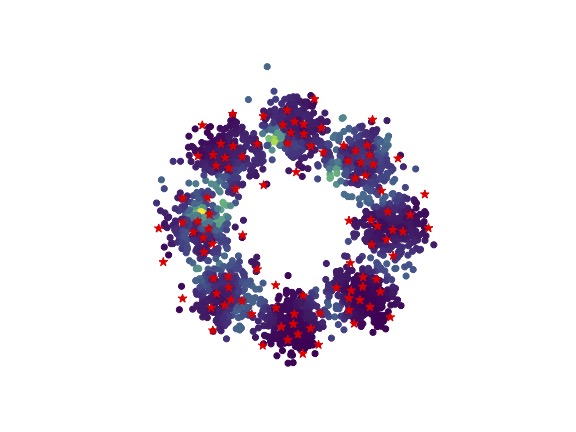}} \\
    \subfigure[Unc.~(Norm, $\tau \rightarrow 0$), Initial]{\includegraphics[clip=True,trim=120 60 100 60,width=0.32\textwidth]{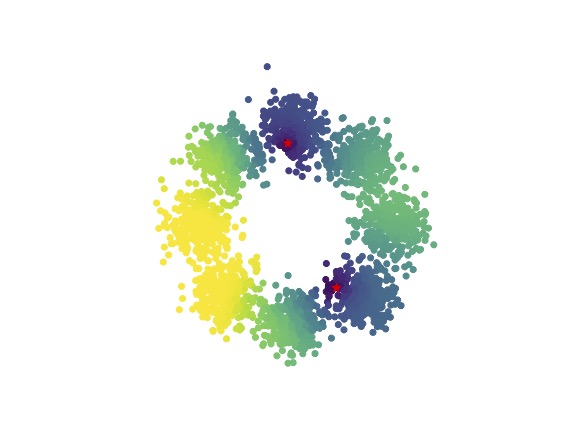}}
    \subfigure[Unc.~(Norm, $\tau \rightarrow 0$), Iter 9]{\includegraphics[clip=True,trim=120 60 100 60,width=0.32\textwidth]{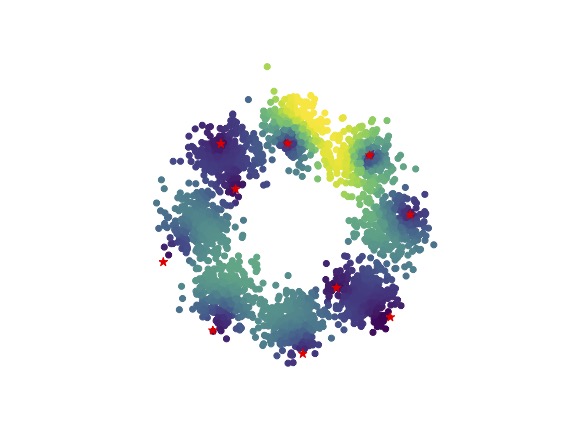}}
    \subfigure[Unc.~(Norm, $\tau \rightarrow 0$), Iter 100]{\includegraphics[clip=True,trim=120 60 100 60,width=0.32\textwidth]{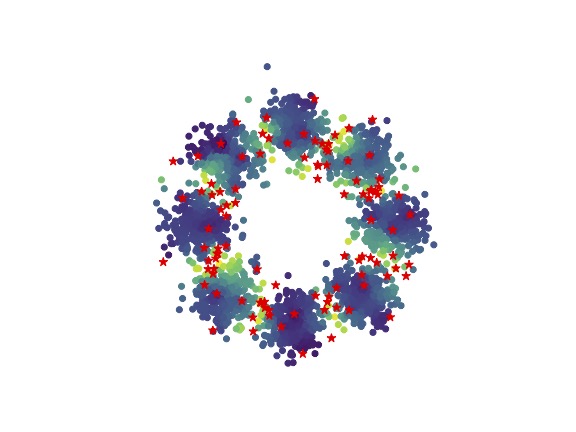}}
    \caption{Acquisition Function Values for Unc.~(SM), Unc.~(Norm), and Unc.~(Norm, $\tau \rightarrow 0$) at different stages of the \textbf{Blobs} experiment. Labeled points are marked as red stars and brighter regions of the heatmap indicate higher acquisition function values.}
    \label{fig:unc-blobs-combined}
\end{figure}


The main observation from this experiment is how poorly Unc.~(SM) performs, as it only attains an overall accuracy of roughly 62\% as the average over the trials. In Figure \ref{fig:unc-blobs-combined}(a-c), we show one trial's acquisition function values heatmap at three different stages of the active learning process using Unc.~(SM). We observe that the active learning queries have been primarily focused on the boundaries between a few clusters, while missing other clusters completely. At each iteration, the heatmap of acquisition function values has only focused on the current classifier's decision boundary which can lead to missing such clusters. In essence, we would qualify the behavior here as ``premature exploitation'', prior to proper exploration of the dataset. 


In contrast, Figures \ref{fig:unc-blobs-combined} (d-i) demonstrate how the ``minimum norm'' uncertainty acquisition functions properly explore the extent of the geometric clustering structure. Both have sampled from every cluster in the ring. It is instructive to further see though that Unc.~(Norm)---which employs a fixed value of $\tau > 0$ at every iteration---has not sampled more frequently \textit{between} clusters by the end of the trial. We may characterize this behavior as not transitioning to proper exploitation of cluster boundaries. On the other hand, in Figure \ref{fig:unc-blobs-combined}(i), we see that by using this minimum norm uncertainty sampling \textit{with decaying values of $\tau \rightarrow 0$} we more frequently sample at the proper cluster boundaries after having sampled from each cluster.


\subsubsection{Box experiment} \label{sec:box-experiment}

\begin{figure}
    \centering
    \subfigure[Ground Truth]{\includegraphics[clip=true,trim= 120 60 120 60,width=0.3\textwidth]{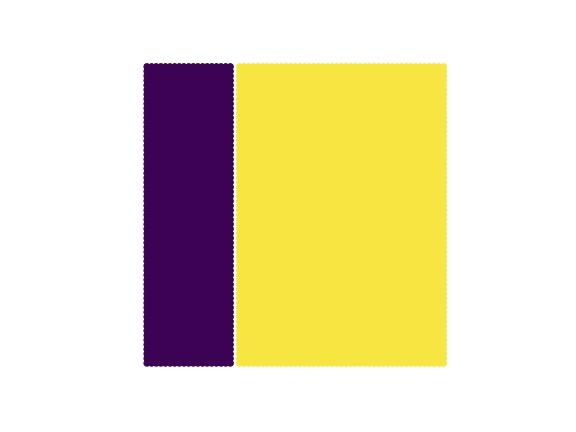}} 
    \hspace{8mm}
    \subfigure[Accuracy Results]{\includegraphics[width=0.5\textwidth]{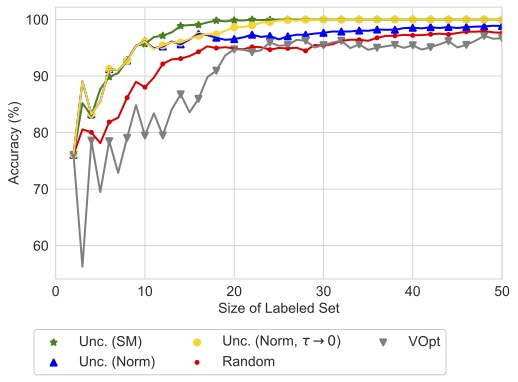}}
    \caption{Ground Truth (a) and Accuracy Results (b) for \textbf{Box} experiment. Notice that Unc.~(Norm) achieves suboptimal overall accuracy. We show in Figure \ref{fig:unc-box-combined}(f) that the distribution of query points later in the active learning process reflect a lack of transition to exploitation.}
    \label{fig:gt-box}
\end{figure}

The \textbf{Box} dataset is simply a 65 $\times$ 65 lattice of points on the unit square, with removing points that lie within a thin, vertical band centered at $x = 0.3$ which also defines the class boundary line (Figure \ref{fig:gt-box}). In contrast to the \textbf{Blobs} experiment, the \textbf{Box} experiment illustrates the need to transition from exploration to exploitation, and how this is accomplished by decreasing $\tau \rightarrow 0$.  In the accuracy plot (Figure \ref{fig:gt-box}(b)), notice how the accuracy achieved by Unc.~(Norm) levels off at a \textit{lower} overall accuracy than both Unc.~(SM) and Unc.~(Norm $\tau \rightarrow 0$). Figure \ref{fig:unc-box-combined} demonstrates that this is due to ``over exploration'' of the dataset instead of transitioning to refining the decision boundary between classes. Active learning seeks to balance exploration versus exploitation while still being sample efficient, making as few active learning queries as possible.

\begin{figure}
    \centering
    \subfigure[Unc.~(SM), Initial]{\includegraphics[clip=true,trim= 120 60 120 60,width=0.3\textwidth]{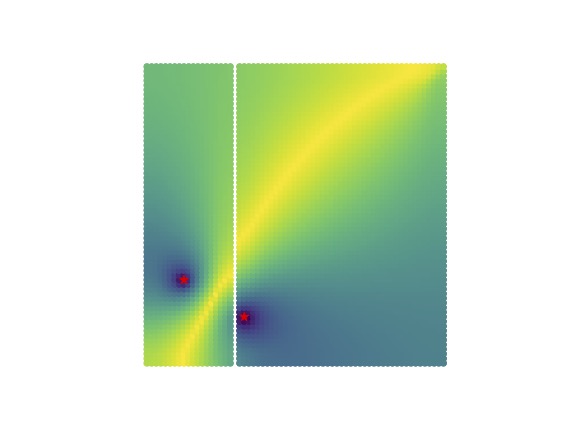}} 
    \hspace{2mm}
    \subfigure[Unc.~(SM), Iter 15]{\includegraphics[clip=true,trim= 120 60 120 60,width=0.3\textwidth]{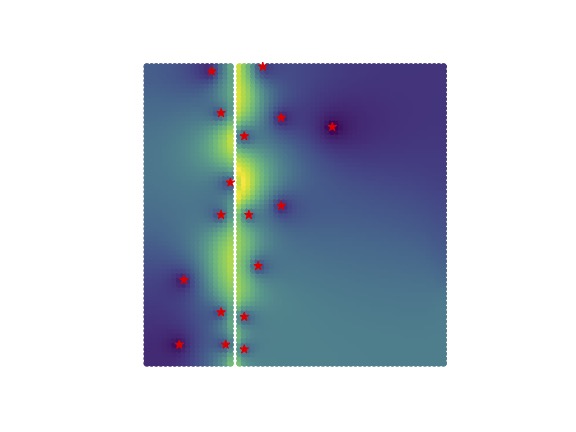}}
    \hspace{2mm}
    \subfigure[Unc.~(SM), Iter 50]{\includegraphics[clip=true,trim= 120 60 120 60,width=0.3\textwidth]{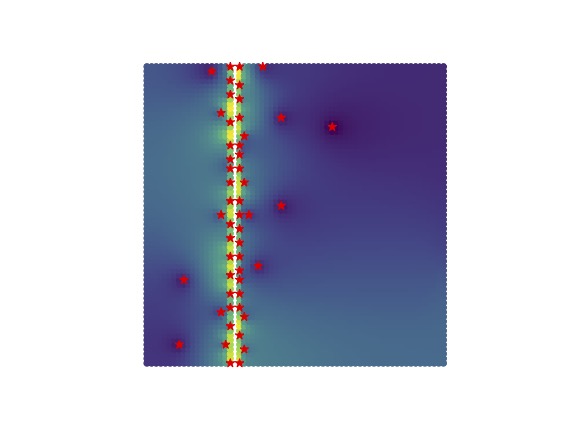}} \\
    \subfigure[Unc.~(Norm), Initial]{\includegraphics[clip=true,trim= 120 60 120 60,width=0.3\textwidth]{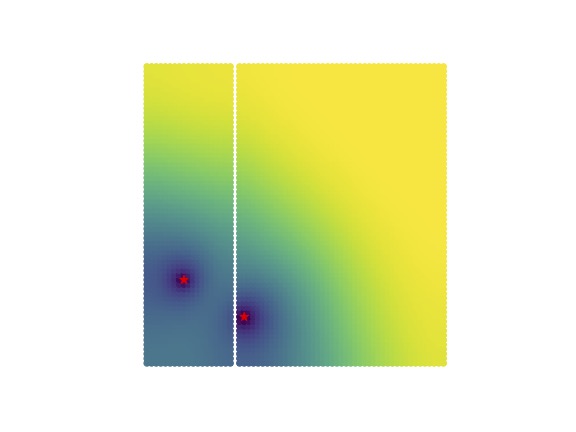}} 
    \hspace{2mm}
    \subfigure[Unc.~(Norm), Iter 15]{\includegraphics[clip=true,trim= 120 60 120 60,width=0.3\textwidth]{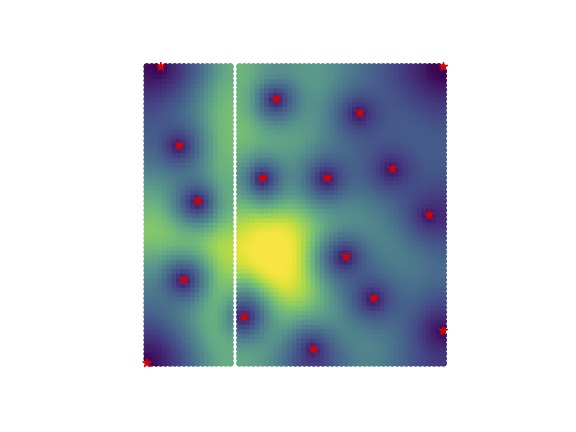}}
    \hspace{2mm}
    \subfigure[Unc.~(Norm), Iter 50]{\includegraphics[clip=true,trim= 120 60 120 60,width=0.3\textwidth]{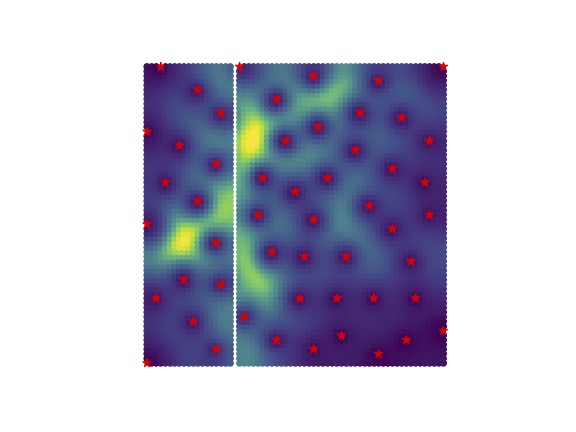}} \\
    \subfigure[Unc.~(Norm, $\tau \rightarrow 0$), Initial]{\includegraphics[clip=true,trim= 120 60 120 60,width=0.3\textwidth]{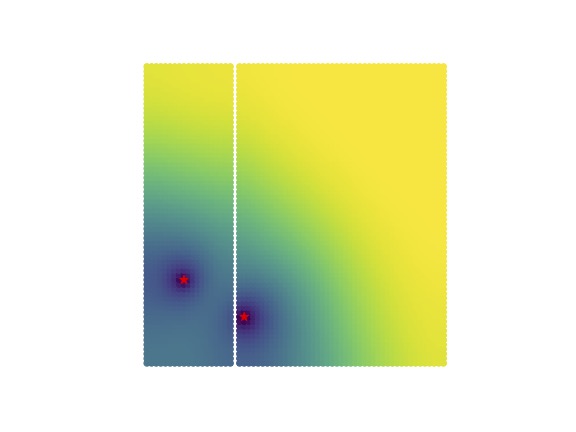}}
    \hspace{2mm}
    \subfigure[Unc.~(Norm, $\tau \rightarrow 0$), Iter 15]{\includegraphics[clip=true,trim= 120 60 120 60,width=0.3\textwidth]{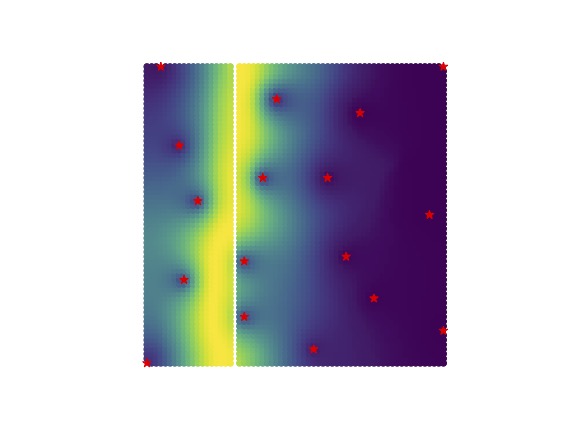}}
    \hspace{2mm}
    \subfigure[Unc.~(Norm, $\tau \rightarrow 0$), Iter 50]{\includegraphics[clip=true,trim= 120 60 120 60,width=0.3\textwidth]{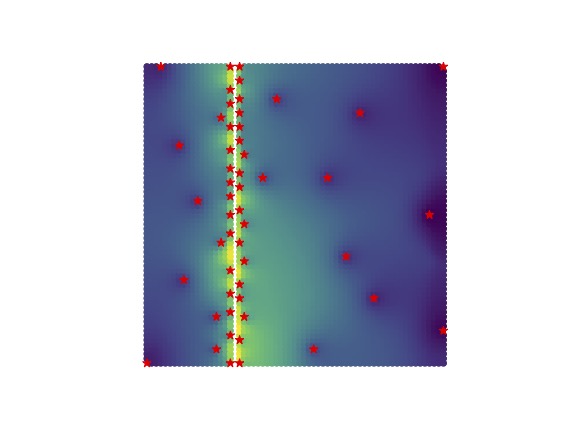}}
    \caption{Acquisition Function Values for Unc.~(SM), Unc.~(Norm), and Unc.~(Norm, $\tau \rightarrow 0$) at different stages of the \textbf{Box} experiment. Labeled points are marked as red stars and brighter regions of the heatmap indicate higher acquisition function values.}
    \label{fig:unc-box-combined}
\end{figure}


As shown in Figures \ref{fig:unc-box-combined} (a-f), both Unc.~(SM) and Unc.~(Norm, $\tau \rightarrow 0$) more efficiently sample the decision boundary between the two classes in this \textbf{Box} dataset. Due to the very simple structure of the dataset, purely exploiting decision boundary information---as done by Unc.~(SM)---is optimal. In contrast, Unc.~(Norm, $\tau \rightarrow 0$) ensures to sparsely explore the extent of the right side of the box \textit{prior to} exploiting the decision boundary. This is due to the decreasing value of $\tau$ over the iterations, and allows for a straightforward transition between exploration and exploitation. We set the value of $K=8$ for the $tau$-decay schedule so that by 8 active learning queries we have transitioned to exploitation. 



\subsubsection{Overall observations}

From the toy experiments presented in Sections \ref{sec:blobs-experiment} and \ref{sec:box-experiment}, we see that the minimum norm uncertainy sampling \textit{with decaying values of $\tau$} has the desired behavior for a sample-efficient criterion that both explores and exploits during the active learning process. Ensuring this behavior in uncertainty sampling is also desirable because of the relatively light computational complexity that uncertainty sampling incurs. We now demonstrate on more complicated, ``real-world'' datasets the effectiveness of minimum norm uncertainty sampling in graph-based active learning.

\subsection{Isolet example} \label{sec:isolet-results}

We demonstrate in this section that minimum norm uncertainty sampling in the PWLL-$\tau$ model overcomes the previously negative results that have characterized uncertainty sampling. In \cite{ji_variance_2012}, the authors introduced the Variance Optimization (i.e., VOpt) acquisition function which quantifies how much unlabeled points would decrease the variance of the conditional distribution over Laplace learning node functions. They showcased this acquisition function on the Isolet spoken letter dataset\footnote{Accessed via \url{https://archive.ics.uci.edu/ml/datasets/isolet}.} from the UCI repository \cite{uci}, which contains 26 different classes. They compared against smallest margin uncertainty sampling (Unc.~(SM)) among other acquisition functions. Of particular interest to the current work is how poorly Unc.~(SM) performed on this task, resulting in significantly worse accuracies than even random sampling.\footnote{We refer the reader to original paper \cite{ji_variance_2012} for more details.} We demonstrate that similar---even superior---performance can be attained on this task by simply using this minimum norm uncertainty sampling (Unc.~(Norm)) that is more appropriate for low-label rate active learning. 

\begin{figure}[!h]
    \centering
    \subfigure[Accuracy]{\includegraphics[width=0.49\textwidth]{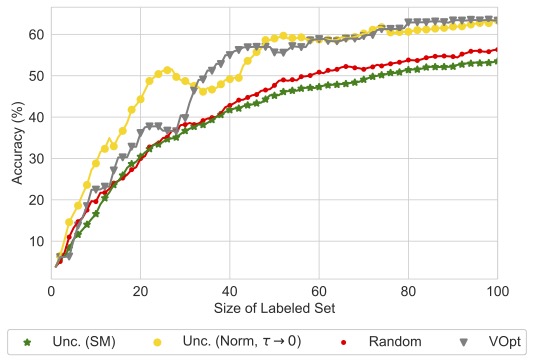}}
    \subfigure[Cluster Proportion]{\includegraphics[width=0.49\textwidth]{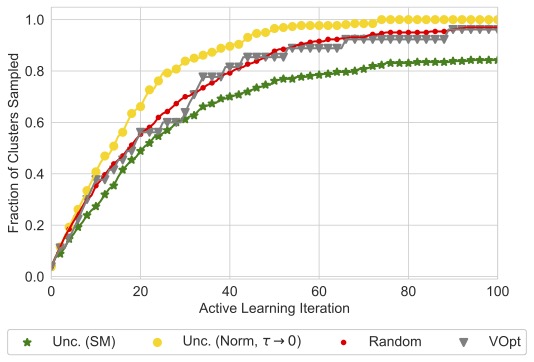}}
    \caption{Accuracy Results (a) and Cluster Proportion (b) plots for \textbf{ISOLET} dataset. Accuracies shown here are in the original Laplace learning model \cite{zhu_semi-supervised_2003} for a more direct comparison with the results from \cite{ji_variance_2012}.}
    \label{fig:isolet-results}
\end{figure}

In Figure \ref{fig:isolet-results}(a), we plot the accuracy results of an active learning test that mimics the setup of the full Isolet dataset (with 26 classes) as described in \cite{ji_variance_2012}. In addition to recreating the results of their test, we have added the results of using Unc.~(Norm, $\tau \rightarrow 0$) for comparison. Please note that the accuracies reported here are in the original Laplace learning model of \cite{zhu_semi-supervised_2003}, not the reweighted Laplace learning \cite{calder2020properly} model that is the focus of the rest of the paper and experimental results. We only add the result of Unc.~(Norm, $\tau \rightarrow 0$) to allow for clearer plots, as Unc.~(Norm) performed nearly identically to Unc.~(Norm, $\tau \rightarrow 0$). 

Each trial (out of 10 total trials) for an acquisition function begins with only a single initially labeled point and 100 query points are thereafter selected sequentially. Thus, only one class has been sampled from at the start of each trial. In Figure \ref{fig:isolet-results}(b), we report the average fraction of ``clusters'' that have been sampled by each iteration of the active learning process. In this case, we treat each individual class as a different cluster. Such a plot demonstrates the explorative capabilities of the acquisition functions as applied to this dataset.

Similar to the test reported in \cite{ji_variance_2012}, smallest margin uncertainty sampling (Unc.~(SM)) performs very poorly at this task, both in terms of accuracy and cluster exploration. Our proposed minimium norm uncertainty sampling (Unc.~(Norm, $\tau \rightarrow 0$)), however, outperforms VOpt in terms of cluster exploration and provides very similar accuracy results. As another point of comparison, the calculation of VOpt requires either an eigendecomposition  or a full inversion of the graph Laplacian matrix, whereas Unc.~(Norm, $\tau \rightarrow 0$) merely requires the current output of the reweighted Laplace learning model. These results provide encouraging evidence for the utility of the proposed method of uncertainty sampling in this current work.

\subsection{Larger datasets} \label{sec:larger-datasets}

In this section, we present the results of active learning experiments for multiclass classification problems derived from the \textbf{MNIST} \cite{lecun-mnisthandwrittendigit-2010}, \textbf{FASHIONMNIST} \cite{xiao2017fashionmnist}, and \textbf{EMNIST} datasets \cite{cohen2017emnist}. We construct similarity graphs for each of these datasets by first embedding the points via the use of variational autoencoders (VAE) \cite{kingma2013auto, kingma_introduction_2019} that were previously trained\footnote{The representations for \textbf{MNIST} and \textbf{FASHIONMNIST} are available in the GraphLearning package \cite{graphlearning}, while the code used to train the VAE for \textbf{EMNIST} is available in our Github repo \url{https://github.com/millerk22/rwll_active_learning}.} in an unsupervised fashion, similar to \cite{calder_poisson_2020}. 

Since a main crux of the present work is to ensure \textit{both} exploration of clusters in a dataset and exploitation of cluster boundaries, we adapt the classification structure of the above datasets to require both. That is, we take the ``true'' class labelings $y_i \in \{0, 1, \ldots, C\}$ (e.g. digits 0-9 for \textbf{MNIST}) and reassign them to one of $k < C$ classes by taking $y_i^{new} \equiv y_i \operatorname{mod} k$; see Table \ref{table:mod-classes} below. 

\begin{table*}[h!]\centering \label{table:mod-classes}
\ra{1.3}
\begin{tabular}{@{}cccccc@{}}\toprule
Resulting Mod Class & 0 & 1 & 2 & 3 & 4 \\
\midrule 
\textbf{MNIST} & 0,3,6,9 & 1,4,7 & 2,5,8 & - & -\\
\textbf{FASHIONMNIST} & 0,3,6,9 & 1,4,7 & 2,5,8 & - & -\\
\textbf{EMNIST} & 0,5,\ldots,45 & 1,6,\ldots,46 & 2,7,\ldots,42 & 3,8,\ldots,43 & 4,9,\ldots,44\\
\bottomrule
\end{tabular}
\caption{Mapping of ground truth class label to $\operatorname{mod} k$ labeling for experiments of Section \ref{sec:larger-datasets}. Each ground truth class, is interpreted as a different ``cluster'' and the resulting class structure for the experiments have multiple clusters per class. For \textbf{MNIST} and \textbf{FASHIONMNIST}, there 10 total ground truth classes and we take labels modulo $k=3$. For \textbf{EMNIST}, there are 47 total ground truth classes and we take labels modulo $k=5$.}
\end{table*}

For each trial of an acquisition function, we select one initially labeled point per \textit{``modulo''} class; therefore, only a subset of ``clusters'' (i.e., the original true classes) has an initially labeled point. In order to perform active learning successfully in these experiments, query points chosen by the acquisition function over the trial must sample from each cluster. In this way, we have created an experimental setup with high-dimensional datasets with potentially more complicated clustering structures wherein we test and compare the following acquisition functions: Uncertainty Sampling (SM), Unc.~(Norm), Unc.~(Norm, $\tau \rightarrow 0$), Random, VOpt \cite{ji_variance_2012} (see Remark \ref{remark:vopt}), $\Sigma$-Opt \cite{ma_sigma_2013} (also see Remark \ref{remark:vopt}), and MCVOpt \cite{miller_efficient_2020}. We perform 10 trials for each acquisition function, where each trial begins with a different initially labeled subset. To clarify, trials begin with only 3 labeled points in the \textbf{MNIST} and \textbf{FASHIONMNIST} experiments and with only 5 labeled points in the \textbf{EMNIST} experiments.

In the left panel of Figures \ref{fig:mnist-results}-\ref{fig:emnist-results}, we show the accuracy performance of each acquisition function averaged over the 10 trials. The right panels of each of these figures display the average proportion of clusters that have been sampled by the acquisition functions at each iteration of the active learning process. We refer to these plots as ``Cluster Exploration'' plots since they directly assess the explorative capabilities of the acquisition functions in question. 

We observe that across these experiments, both Unc.~(Norm) and Unc.~(Norm, $\tau \rightarrow 0$) consistently achieve the best accuracy and cluster exploration results. It is somewhat surprising that without decaying $\tau$, the Unc.~(Norm) acquisition function seems to perform the best even after each cluster has been explored. The experiments in Section \ref{sec:toy-experiments} suggest that the optimal performance in the exploitation phase of active learning would require taking $\tau \rightarrow 0$. We hypothesize that the clustering structure of high-dimensional data---like these datasets---is much more complicated than our intuition would suggest from analyzing toy and other visualizable (i.e., 1D, 2D, or 3D) datasets. Regardless, we see that minimum norm uncertainty acquisition function consistently outperforms other acquisition functions in these low-label rate active learning experiments. We emphasize again here that the computational cost of uncertainty sampling acquisition functions make them especially useful for active learning.

\begin{remark} \label{remark:vopt}
    Due to the large nature of these datasets, computing the original VOpt and $\Sigma$-Opt criterions are inefficient (and often intractable) since it requires computing the inverse of a perturbed graph Laplacian matrix; this inverse is dense and burdensome to store in memory. We initially used an approximate criterion that utilizes a subset of eigenvalues and eigenvectors of the graph Laplacian, similar to what was done in \cite{miller_model-change_2021}. However, we noticed significantly poor results on the \textbf{MNIST} and \textbf{FASHIONMNIST} experiments seemingly due to the spectral truncation with a resulting oversampling of a single cluster during the active learning process.  
    
    As an alternative to the spectral truncation, we performed a full calculation of these acquisition functions on a small, random subset of $500$ unlabeled points at each active learning iteration. This performed significantly better than the spectral truncation in these two experiments, and so we report their performance in this section. In Figures \ref{fig:mnist-results} and \ref{fig:fashionmnist-results} we refer to this adaptation with the suffix ``(Full)''; e.g., we name its results by ``VOpt (Full)''. The small choice of unlabeled points on which to evaluate the acquisition function in the full setting is due to the burdensome computation needed at each step that scales with the size of this subset; at this reported choice of $500$ points each active learning iteration already takes roughly 6 minutes to complete. Due to its even greater size, we do not perform this computation on the \textbf{EMNIST} dataset, and furthermore the performance of the approximate VOpt already achieves comparable accuracy to the other reported methods in this dataset.
\end{remark}

\begin{figure}
    \centering
    \subfigure[Accuracy]{\includegraphics[width=0.49\textwidth]{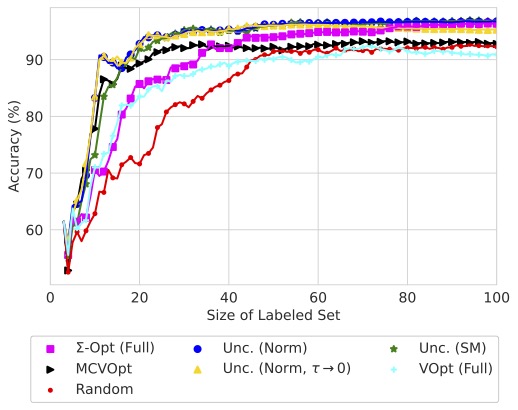}}
    \subfigure[Cluster Proportion]{\includegraphics[width=0.49\textwidth]{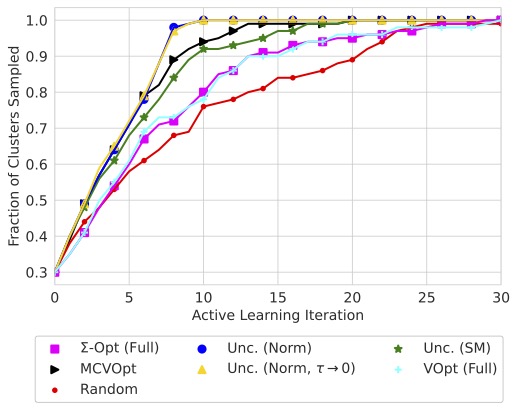}}
    \caption{Accuracy Results (a) and Cluster Proportion (b) plots for \textbf{MNIST} dataset.}
    \label{fig:mnist-results}
\end{figure}

\begin{figure}
    \centering
    \subfigure[Accuracy]{\includegraphics[width=0.49\textwidth]{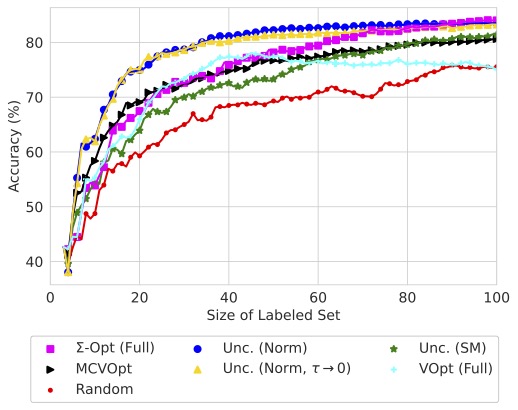}}
    \subfigure[Cluster Proportion]{\includegraphics[width=0.49\textwidth]{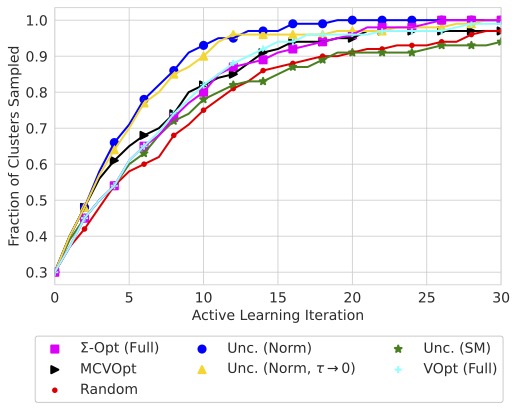}}
    \caption{Accuracy Results (a) and Cluster Proportion (b) plots for \textbf{FASHIONMNIST} dataset. }
    \label{fig:fashionmnist-results}
\end{figure}

\begin{figure}
    \centering
    \subfigure[Accuracy]{\includegraphics[width=0.49\textwidth]{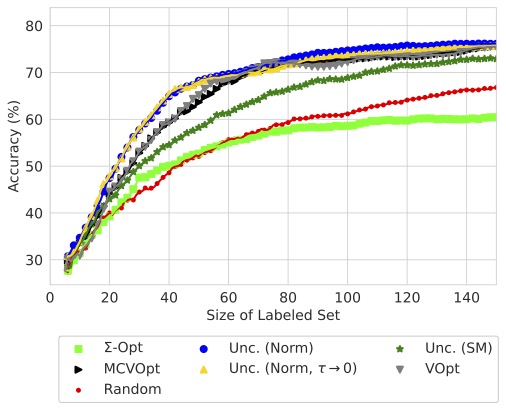}}
    \subfigure[Cluster Proportion]{\includegraphics[width=0.49\textwidth]{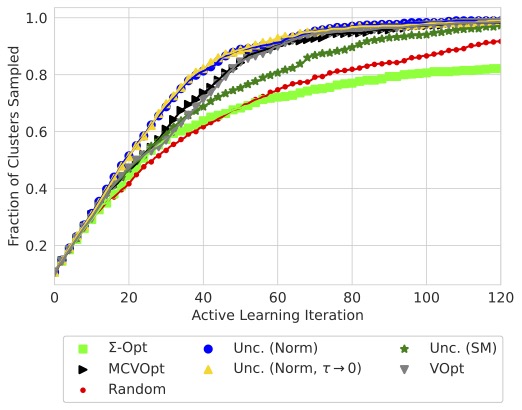}}
    \caption{Accuracy Results (a) and Cluster Proportion (b) plots for \textbf{EMNIST} dataset. }
    \label{fig:emnist-results}
\end{figure}



\section{Continuum analysis of active learning} \label{sec:theory}

We now study our active learning approach rigorously through its continuum limit. As was shown in \cite{calder2020properly}, the continuum limit of \eqref{eq:rw-lap-learning} is the family of singularly weighted elliptic equations
\begin{equation}\label{eq:rw_lap_continuum}
\left\{
\begin{aligned}
\tau u_i - \rho^{-1}\div\left(\gamma\rho^2  \nabla u_i \right) &= 0,&& \text{in } \Omega \setminus \L  \\
u_i &= 1,&& \text{on } \L_i\\
u_i&= 0,&& \text{on } \L\setminus \L_i,
\end{aligned}
\right.
\end{equation}
where $\rho$ is the density of the datapoints, $\gamma$ is the singular reweighting, described in more detail below, $\L_i\subset \Omega$ are the labeled points in the $i^{\rm th}$ class, and $\L = \cup_{i=1}^C \L_i$ the locations of all labeled points. The notation $\nabla$ refers to the gradient vector and $\div$ is the divergence. The solutions $u_i$ also satisfy the homogeneous Neumann boundary condition $\nabla u \cdot \nu = 0$ on $\partial \Omega$, where $\nu$ is the outward unit normal vector to $\Omega$, but we omit writing this as it is not directly used in any of our arguments.  We assume the sets $\L_i$ are all finite collections of points. The classification decision for any point $x\not\in \L$ is given by
\[\ell(x) = \argmax_{1 \leq i \leq C} u_i(x).\]
The continuum version of the uncertainty sampling acquisition function is then given by
\begin{equation}\label{eq:acq_cont}
\A(x) =  \sqrt{u_1(x)^2 + u_2(x)^2 + \cdots + u_C(x)^2}.
\end{equation}
The aim in this section is to use continuum PDE analysis to rigorously establish the exploration versus exploitation tradeoff in uncertainty norm sampling, and illustrate how it depends on the choice of the decay parameter $\tau$.

\subsection{Illustrative 1D continuum analysis} \label{sec:1d-theory}

We proceed at first with an analysis of the continuum equations \eqref{eq:rw_lap_continuum} in the one dimensional setting, where the equations are ordinary differential equations (ODEs). The analysis is straightforward and the reweighting \eqref{eq:gamma} is no longer necessary for well-defined continuum equations with finite labeled data. The conclusions are insightful for the subsequent generalization to higher dimensions in Section \ref{sec:exploration}.


Consider an interval $\Omega = (x_{min}, x_{max}) \subset \mbb R$ with density $0 < \rho_{min} \le \rho(x) \le \rho_{max} < +\infty$.
Assume a binary classification structure on this dataset, and further assume we have been given at least one labeled point per class. Let the pairs $\{(x_i, y_i)\}_{i=1}^\ell \subset \Omega \times \{0,1\}$ be the input-class values for the currently labeled points. Without loss of generality, let us assume that the indexing on these labeled points reflects their ordering in the reals; namely, $x_i < x_{i+1}$ for each $i \le \ell -1$. For ease in our discussion, we also assume that $x_1 = x_{min}$ and $x_\ell = x_{max}$, the endpoints of the domain (see Figure \ref{fig:1d-init}).

\begin{figure}[h]
    \centering
    \includegraphics[width=0.9\textwidth]{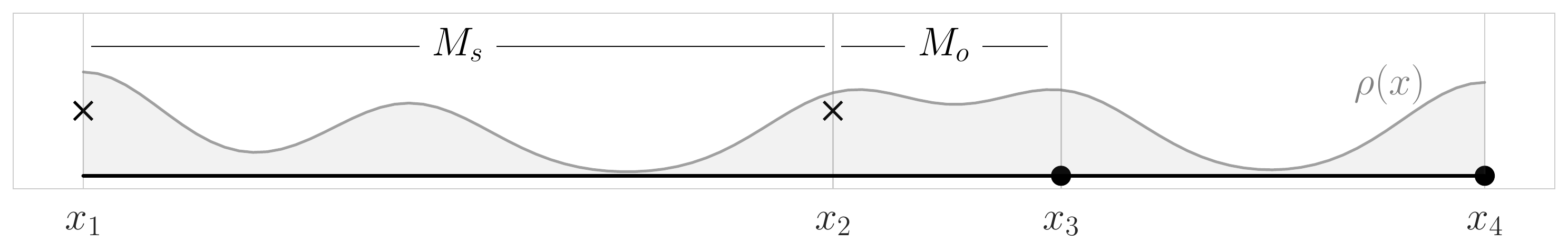}
    \vspace{-1em}
    \caption{Visualization of 1D continuum example setup. The density $\rho(x)$ is plotted in gray, while the labeled points $x_1, x_2, x_3, x_4$ are plotted where the corresponding label is denoted by $\times$ or a solid dot. $\R_s$ marks the length between two similarly labeled points, while $\R_o$ marks the length between two oppositely labeled points.}
    \label{fig:1d-init}
\end{figure}

Solving the PWLL-$\tau$ equation\footnote{Without the reweighting \eqref{eq:gamma} due to the simple geometry in one dimension.} \eqref{eq:rw_lap_continuum} on $\Omega$ can be broken into a number of subproblems defined on the intervals $(x_1, x_2), \ldots, (x_{\ell-1}, x_{\ell}) \subset \mbb R$, with boundary conditions determined by the corresponding labels of the endpoints $x_i$. There are three separate kinds of sub-problems that need to be solved, as determined by these boundary conditions, that we will term (1) the \textit{oppositely labeled problem} (when $y_i \not= y_{i+1}$) and (2) the \textit{similarly labeled problem} (when $y_i = y_{i+1}$). 

Given the current state of the labeled data, the active learning process selects a new query point $x^\ast = \argmin_{x \in \Omega} \ \mcl A(x)$ via the minimum norm acquisition function \eqref{eq:acq_cont}. 
We can quantify the explorative behavior of our acquisition function \eqref{eq:acq_cont} by comparing the minimizers of $\mcl A(x)$ in the different subintervals $(x_i, x_{i+1})$. Due to the simple geometry of the problems in one dimension, our analysis reduces to a pairwise comparison of $\mcl A(x)$ on (i) an interval of length $\R_o$ between \textit{oppositely labeled points} and (ii) an interval of length $\R_s$ between \textit{similarly labeled points}. Defining $\mcl A_s $ and $\mcl A_o$ as the acquisition function on the respective oppositely and similarly labeled problem subintervals, we compare the values $\min \mcl A_o(x)$ and $\min \mcl A_s(x)$ on said subintervals.

In this simple one-dimensional problem, we may characterize ``explorative'' query points as residing in relatively \textit{large} intervals between labeled points, regardless of the labels of the endpoints. Conversely, we characterize ``exploitative'' query points as residing between \textit{oppositely labeled points that are close together}. In Figure \ref{fig:1d-init}, exploration would correspond to sampling in $(x_1,x_2)$ or $(x_3,x_4)$, while exploitation would correspond to sampling in $(x_2, x_3)$. 

The acquisition function \eqref{eq:acq_cont} is directly a function of the magnitudes of the solutions to \eqref{eq:rw_lap_continuum} with the corresponding boundary conditions. Due to the boundary conditions intervals between oppositely labeled points, the solutions to \eqref{eq:rw_lap_continuum} necessarily interpolate between $0$ and $1$ along the interval (see Figure \ref{fig:solutions-acqfuncs}(a)). In the oppositely labeled problem, however, there is only decay in the solution $v_0$ that solves \eqref{eq:rw_lap_continuum} with labels $y(x_i) = y(x_{i+1}) = 1$ when $\tau >0$, and the extent of this decay is controlled by the size of $\tau$, the length of the interval $\R_s$, and the density $\rho$ on the interval. As such, we identify how $\tau$ must be chosen in order to produce small acquisition function values between similarly labeled points in relatively large regions as compared to large values in relatively small regions between oppositely labeled points. 


\subsubsection{Exploration guarantee in one dimension}
\begin{figure}[t]
    \centering
    \subfigure[Oppositely Labeled]{\includegraphics[height=12em]{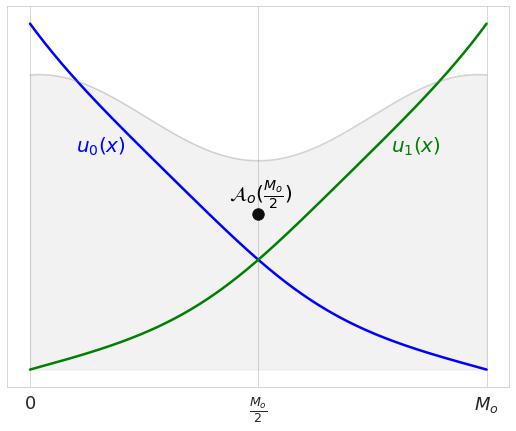}} 
    \hspace{3em}
    \subfigure[Similarly Labeled]{\includegraphics[height=12em]{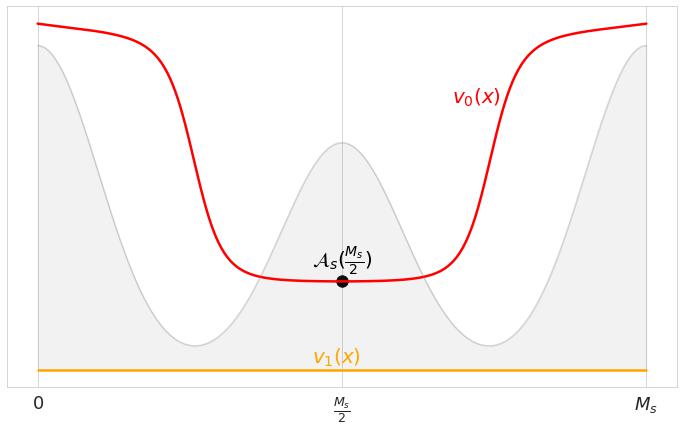}} 
    \caption{Visualization of one-dimensional solutions to \eqref{eq:rw_lap_continuum} in oppositely (a) and similarly (b) labeled regions. Solutions $u_0,u_1$ (blue, green lines) to \eqref{eq:rw_lap_continuum} are shown in panel (a) when endpoints have opposite labels, while $v_0,v_1$ (red, orange lines) to \eqref{eq:rw_lap_continuum} are shown in panel (b) when endpoints have the same labels. The background density $\rho$ in each respective region is shown in gray, and the acquisition function value at the midpoint of the interval is shown as a black dot. The minimum acquisition function value occurs at the midpoint if the density is symmetric, which we show here for simpler presentation.  }
    \label{fig:solutions-acqfuncs}
\end{figure}

In Supplemental Material Section \ref{smsec:warmup-const-density}, we first derive an explicit condition in the case that the density $\rho(x) \equiv \rho$ is constant to guarantee that $\min \mcl A_s(x) < \min \mcl A_o(x)$. As long as the length between oppositely labeled points ($\R_o$) is \textit{small} enough compared to the length between similarly labeled points ($\R_o$), 
then this gives the condition that the quantity $\frac{\tau\R^2_s}{\rho}$ must be relatively \textit{large}.  We can then generalize this result to cases when the density $\rho(x)$ is no longer constant, but rather obeys some mild assumptions. Namely, we give the mild assumption that the density $\rho(x)$ (i) is sufficiently smooth, (ii) is \textit{symmetric about the midpoint of the interval} (see Assumption \ref{assumption:symmetry}) between similarly labeled points, and (iii) obeys \textit{a bounded derivative condition at the ends of the interval} (see Assumption \ref{assumption:end-intervals}) between oppositely labeled points. Under these mild assumptions we give the following simplified guarantee on exploration, which we prove rigorously in Section \ref{smsec:compare-1d-bounds}.
\begin{proposition}[Simplified version of Proposition \ref{smprop:1d-result}] \label{prop:1d-result}
    Suppose that the density $\rho(x)$ satisfies Assumption \ref{assumption:end-intervals} in the oppositely labeled problem region and Assumption \ref{assumption:symmetry} in the similarly labeled problem region. Let the interval length $\R_o$ be relatively small compared to $\R_s$; i.e., $\R_o = \beta \R_s$ for some $\beta \le \frac{1}{4}$. Then we are ensured that 
    \[
        \min_{x}\ \mcl A_s(x) < \min_{x } \ \mcl A_o(x)
    \]
    as long as $\tau > 0$ and $\R_s$ jointly satisfy the following inequality
    \begin{equation} \label{eq:tau-condition-messy-simplifieid}
    \R_s^2 \lp C_0(\rho_s) \sqrt{\tau} -  C_1(\rho_o) \beta^2 \tau  \rp \ge  8\ln 2,
    \end{equation}
    where $C_0(\rho_s)$ and $C_1(\rho_o)$ are constants that depend on the density $\rho$ on the similarly and oppositely labeled intervals, respectively denoted $\rho_s$ and $\rho_o$.
\end{proposition}

Figure \ref{fig:solutions-acqfuncs} illustrates the main idea of Proposition \ref{prop:1d-result}. As long the similarly labeled region (Figure \ref{fig:solutions-acqfuncs}(b) has significantly large regions where the density $\rho(x)$ is sufficiently small compared to the oppositely labeled region (Figure \ref{fig:solutions-acqfuncs}(a)), then we can be assured that choosing $\tau > 0$ large enough will result in query points between similarly labeled points that are relatively far from each other. 

This relationship is also summarized in the inequality \eqref{eq:tau-condition-messy-simplifieid}, which highlights that larger $\tau$ and interval length $\R_s$ are necessary in order to satisfy said inequality. This inequality simply quantifies how $\tau>0$ must be chosen in order to select query points between similarly labeled points that are relatively far away from each other; i.e., to ensure exploration of such regions that would otherwise be missed if $\tau$ we not sufficiently large. The effect of the relative ratio of the intervals, $\beta$, is also highlighted in \eqref{eq:tau-condition-messy-simplifieid}; namely, the smaller the region between oppositely labeled points the easier it is to satisfy this inequality. Intuitively one can see that if $\beta$ is not small, then the region between oppositely labeled points is relatively large and it will be more difficult to satisfy the stated inequality. However, in this case querying between oppositely labeld points that are relatively distant from each other is desirable and would be characterized as explorative.

\subsection{Exploration bounds in arbitrary dimensions}
\label{sec:exploration}

In this section, we show how larger values for $\tau$ lead to explorative behaviour in higher dimensional problems. In particular, we show that the acquisition function $\A(x)$ is small on unexplored clusters, and large on sufficiently well-explored clusters. This ensures that adequate exploration occurs before exploitation.

Let us remark that the reweighting term $\gamma$ must be sufficiently singular near the labels $\L$ in order to ensure that \eqref{eq:rw_lap_continuum} is well-posed. We recall from \cite{calder2020properly} that we require that $\gamma$ has the form
\begin{equation}\label{eq:gamma}
\gamma(x) = 1 + \dist(x,\L)^{-\alpha},
\end{equation}
where $\alpha > d-2$. In practice, we choose $\gamma$ as the solution of the graph Poisson equation \eqref{eq:gamma_eq_discrete} introduced earlier. To make the analysis in this section tractable, we assume here that $\gamma$ satisfies \eqref{eq:gamma}, as was assumed in \cite{calder2020properly}.
We emphasize here that without the singular reweighting  $\gamma$, the equation \eqref{eq:rw_lap_continuum} is ill-posed when the label set $\L$ is finite, and as such, there is no continuum version of active learning for us to study.


For an open set $A\subset \R^d$  and $r>0$ we define the nonlocal boundary $\partial_r A$ as 
\[\partial_r A = \overline{(A + B_r)} \setminus A.\]
The nonlocal boundary is essentially a tube of radius $r$ surrounding the set $A$. The usual boundary is obtained by taking $r=0$, so $\partial A=\partial_0 A$.

Our first result concerns upper bounds on the acquisition function in an unexplored cluster. 
\begin{theorem}\label{thm:explore_new_cluster}
Let $\tau\geq 0$, $s,R>0$ and $\C \subset \Omega$ with $\partial_{2s} \C \subset \Omega$ and $\L\cap (\C+B_{R+2s})=\varnothing$. Let
\[\delta = \max_{\partial_{2s} \C}\rho.\]
Assume that
\begin{equation}\label{eq:explore_cond}
\sqrt{\frac{\tau}{\delta}} \geq 3\left(\tfrac{d}{s} + 2\|\nabla \log \rho\|_{L^\infty(\partial_s \C)}\right)(1+R^{-\alpha}) + 3R^{-\alpha-1}.
\end{equation}
Then it holds that
\begin{equation}\label{eq:acq_upper_bound}
\sup_{\C}\A \leq  \sqrt{C}\exp\left(-\frac{s}{4}\sqrt{\frac{\tau}{\delta}}\right).
\end{equation}
\end{theorem}
\begin{remark}\label{rem:upper_bound}
Theorem \ref{thm:explore_new_cluster} shows that the acquisition function $\A$ is exponentially small on an unexplored cluster $\C$ provided there is a thin surrounding set $\partial_s \C$ of the cluster on which the density is small (less than $\delta$), relatively smooth (so $\nabla \log \rho$ is not too large), and relatively far away from other labeled datapoints (so that $R$ is not too large). All of these smallness assumptions are relative to the size of the ratio $\tau/\delta$ as expressed in \eqref{eq:explore_cond}. 
\end{remark}

To ensure that new clusters are explored, we also need to lower bound the acquisition function nearby the existing labeled set. To do this, we need to introduce a model for the clusterability of the dataset. Let $\Omega_1,\Omega_2,\dots,\Omega_C\subset \Omega$ be disjoint sets representing each of the $C$ classes in the dataset. We assume that the labels are chosen from the corresponding class sets, so that $\L_i \subset \Omega_i$ for each $i$. We assume there is a positive separation between the classes, measured by the quantity
\begin{equation}\label{eq:cluster_sep}
\s := \min_{i\neq j}\dist(\Omega_i,\Omega_j).
\end{equation}
The definition of $\s$ implies that $(\Omega_i + B_\s)\cap \Omega_j = \varnothing$ for all $i\neq j$. We define the union of the classes as $\Omega'= \cup_{i=1}^C \Omega_i$. We note that we do not have $\Omega'=\Omega$, and it is important that there is room in the background $\Omega\setminus \Omega'$, which provides a separation between classes. The background $\Omega\setminus \Omega'$ may have low density (though we do not assume this below), and can consist of outliers or datapoints that have characteristics of multiple classes and may be hard to classify.

\begin{theorem}\label{thm:explore_labels}
Let $\tau\geq 0$ and $\alpha>d-2$. Assume that $\L_i\subset \Omega_i$ for $i=1,\dots,C$, and let $r>0$ be small enough so that $r \leq \tfrac{1}{4}\s$, 
\begin{equation}\label{eq:rtau2}
\tau r^d \leq \frac{1}{2^d9}(\alpha+2-d)^2\inf_{\Omega' + B_{2r}}\rho,
\end{equation}
and
\begin{equation}\label{eq:rcond2}
4\|\nabla \log \rho\|_{L^\infty(\Omega'+B_{2r})}(1 + 2^\alpha r^{\alpha})r + \alpha 2^\alpha r^{\alpha} \leq  \tfrac{1}{4}(\alpha + 2-d).
\end{equation}
Then we have
\begin{equation}\label{eq:acq_lower}
\inf_{\L + B_r}\A \geq 1 - 2^{-\frac{1}{2}(\alpha + 2-d)}.
\end{equation}
\end{theorem}

\begin{figure}[!t]
\centering
\includegraphics[width=0.7\textwidth]{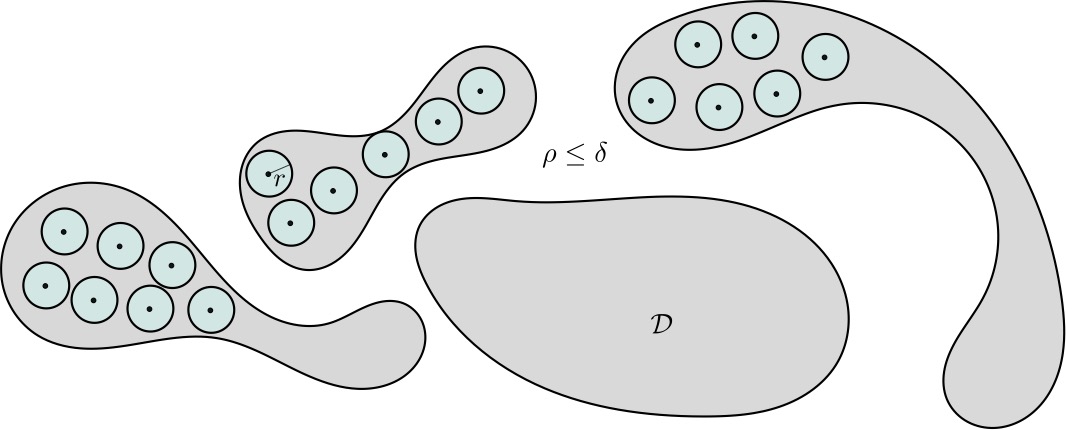}
\caption{Illustration of the implications of Theorems \ref{thm:explore_new_cluster} and \ref{thm:explore_labels}, and the discussion in Remark \ref{rem:explanation}. The gray regions are the 4 clusters of high density in the dataset, and the density is small $\rho \leq \delta$ between clusters. The current labeled set are the points at the centers of the blue balls.  Theorems \ref{thm:explore_new_cluster} and \ref{thm:explore_labels} guarantee that the next labeled point cannot lie in any of the blue balls, which correspond to the dilated label set $\L + B_r$. Once the dilated labels cover the existing clusters, the algorithm is guaranteed to select a point from the unexplored cluster $\C$. The number of labeled points selected from a given cluster during exploration is bounded by its $\frac{r}{2}$-packing number, as explained in Remark \ref{rem:explanation}.  }
\label{fig:clusters}
\end{figure}

\begin{remark}\label{rem:explanation}
Let us make some remarks on the applications of Theorems \ref{thm:explore_new_cluster} and \ref{thm:explore_labels}. First, we note that the choice of $s$ in Theorem \ref{thm:explore_new_cluster} can be made proportional to the separation between clusters $\s$ defined in \eqref{eq:cluster_sep}. We can then choose $\tau$ to ensure \eqref{eq:explore_cond} holds in Theorem \ref{thm:explore_new_cluster}, and choose $r>0$ to satisfy the conditions in Theorem \ref{thm:explore_labels}. These choices are all dependent on the domain, the clusterability assumption, and the density, but are independent of the choices of labeled points $\L_i$. Now, combining Theorems \ref{thm:explore_new_cluster} and \ref{thm:explore_labels} we see the condition 
\begin{equation}\label{eq:acq_cond}
\sqrt{C}\exp\left(-\frac{\s}{4}\sqrt{\frac{\tau}{\delta}}\right)\leq 1 - 2^{-\frac{1}{2}(\alpha + 2-d)}
\end{equation}
is important. Whenever \eqref{eq:acq_cond} holds, the region $\C$ will be explored \emph{before} a new labeled point is chosen within distance $r$ of any existing labeled point. This is exactly the \emph{exploration} property that we desire in an active learning algorithm. In the early stages, the algorithm should seek to explore new clusters, or continue to sufficiently explore existing clusters. The algorithm will not choose another labeled point within distance $r$ of an existing label until the entire dataset is thoroughly explored, at which point the active learning algorithm should switch to exploitation.

In fact, we can make the statements above a little more precise. Whenever $\L_i + B_r \supset \Omega_i$, we have from Theorem \ref{thm:explore_labels} that 
\[\inf_{\Omega_i} \A \geq 1 - 2^{-\frac{1}{2}(\alpha + 2-d)}.\]
In this case, provided \eqref{eq:acq_cond} holds, the algorithm will not select another point from $\Omega_i$ until \emph{all} other cluster have been explored. Since the algorithm also cannot choose a new point within distance $r$ of existing points, the set $\L_i$ is a $r$-net of $\Omega_i$. In particular, the balls $B_{\frac{r}{2}}(z)$ for $z\in \L_i$ are disjoint, so $\L_i+B_{\frac{r}{2}}$ is a $\tfrac{r}{2}$-\emph{packing} of $\Omega_i$. We define an $\epsilon$-\emph{packing} of $\Omega_i$ as a disjoint union of $\epsilon$-balls that are centered at points in $\Omega_i$. Therefore, the maximum number of points in $\L_i$ is given by the $\epsilon$-\emph{packing number} of $\Omega_i$ with $\epsilon=\tfrac{r}{2}$, which is defined as 
\[M(\Omega_i,\epsilon) = \max\left\{m \, : \, \text{there exists an } \epsilon \text{-packing of }\Omega_i \text{ with }m\text{ balls.} \right\}.\]
Thus, Theorems \ref{thm:explore_new_cluster} and \ref{thm:explore_labels} show that our uncertainty norm sampling active learning algorithm, in the continuum, cannot select more than the packing number $M(\Omega_i,\tfrac{r}{2})$ of points from $\Omega_i$ until \emph{all} clusters have been explored. The packing number $M(\Omega_i,\tfrac{r}{2})$ depends on the geometry of the cluster $\Omega_i$ and can be large for clusters that are not spherical (e.g., clusters that are ``thin'' and ``long'' in certain directions). These results are illustrated in Figure \ref{fig:clusters}. 

The reader may have observed there is an implicit assumption made throughout this remark that there are no labeled points selected from the background region $\Omega\setminus \Omega'$. Indeed, if such outlying datapoints are selected as labeled points, then our results do not hold. In practice, one can perform sampling proportional to a density estimation, or simply remove outliers, to avoid such an issue. We discuss how this can be done in Section \ref{sec:kde}, and we have performed experiments with this. We have found that the results are similar with and without the outlier removal, so we see this as an extra step that one has the option of performing in practice, to maximally align the algorithm with the theory, but we do not see it as a necessary step in practice. 
\end{remark}


\section{Conclusion}

We have demonstrated that uncertainty sampling is sufficient for exploration in graph-based active learning by using the norm of the output node function of the PWLL-$\tau$ model as an acquisition function. We provide rigorous mathematical guarantees on the explorative behavior of the proposed acquisition function. This is made possible by the well-posedness of the corresponding continuum limit PDE of the PWLL-$\tau$ model. Our analysis elucidates how the choice of hyperparamter $\tau >0$ directly influences these guarantees; in the one dimensional case this effect is most clearly illustrated. In addition, we provide numerical experiments that further illustrate the effect of both our acquisition function and the hyperparameter $\tau$ on the sequence of active learning query points. Other numerical experiments confirm our theoretical guarantees and demonstrate favorable performance in terms of both accuracy and cluster exploration.

\bibliographystyle{siamplain}
\bibliography{references.bib}

\begin{thebibliography}{10}

\bibitem{balcan2009agnostic}
{\sc M.-F. Balcan, A.~Beygelzimer, and J.~Langford}, {\em Agnostic active
  learning}, J. of Computer and System Sciences, 75 (2009), pp.~78--89,
  \url{https://doi.org/10.1016/j.jcss.2008.07.003}.

\bibitem{balcan_margin_2007}
{\sc M.-F. Balcan, A.~Broder, and T.~Zhang}, {\em Margin based active
  learning}, in International Conference on Computational Learning Theory,
  vol.~4539, Springer Berlin Heidelberg, 2007, pp.~35--50,
  \url{https://doi.org/10.1007/978-3-540-72927-3_5}.

\bibitem{bertozzi_diffuse_2016}
{\sc A.~L. Bertozzi and A.~Flenner}, {\em Diffuse interface models on graphs
  for classification of high dimensional data}, SIAM Review,  (2016),
  \url{https://doi.org/10.1137/16M1070426}.

\bibitem{bertozzi2019graph}
{\sc A.~L. Bertozzi and E.~Merkurjev}, {\em Graph-based optimization approaches
  for machine learning, uncertainty quantification and networks}, in Handbook
  of Numerical Analysis, vol.~20, Elsevier, 2019, pp.~503--531.

\bibitem{cai2017active}
{\sc H.~Cai, V.~W. Zheng, and K.~C.-C. Chang}, {\em Active learning for graph
  embedding}, preprint arXiv,  (2017), \url{https://arxiv.org/abs/1705.05085}.

\bibitem{calder2018game}
{\sc J.~Calder}, {\em The game theoretic p-{Laplacian} and semi-supervised
  learning with few labels}, Nonlinearity, 32 (2018), p.~301.

\bibitem{calder_consistency_2019}
{\sc J.~Calder}, {\em Consistency of {Lipschitz} learning with infinite
  unlabeled data and finite labeled data}, SIAM J. on Mathematics of Data
  Science, 1 (2019), pp.~780--812, \url{https://doi.org/10.1137/18m1199241}.

\bibitem{graphlearning}
{\sc J.~Calder}, {\em {GraphLearning} {Python} package}, Jan. 2022,
  \url{https://doi.org/10.5281/zenodo.5850940}.

\bibitem{calder_poisson_2020}
{\sc J.~Calder, B.~Cook, M.~Thorpe, and D.~Slep{\v{c}ev}}, {\em Poisson
  learning: Graph-based semi-supervised learning at very low label rates}, in
  Proceedings of the 37th {International} {Conference} on {Machine} {Learning},
  Proceedings of Machine Learning Research, Nov. 2020, pp.~1306--1316.

\bibitem{calder2022poisson}
{\sc J.~Calder, B.~Cook, M.~Thorpe, D.~Slep{\v{c}ev}, Y.~Zhang, and S.~Ke},
  {\em Graph-based semi-supervised learning with {P}oisson equations}, In
  preparation,  (2022).

\bibitem{calder2022improved}
{\sc J.~Calder and N.~Garc\'ia~Trillos}, {\em Improved spectral convergence
  rates for graph {Laplacians} on $\varepsilon$-graphs and k-{NN} graphs},
  Applied and Computational Harmonic Analysis, 60 (2022), pp.~123--175,
  \url{https://doi.org/10.1016/j.acha.2022.02.004}.

\bibitem{calder_rates_2020}
{\sc J.~Calder, D.~Slep{\v{c}ev}, and M.~Thorpe}, {\em Rates of convergence for
  {Laplacian} semi-supervised learning with low labeling rates}, preprint
  arXiv,  (2020), \url{http://arxiv.org/abs/2006.02765}.

\bibitem{calder2020properly}
{\sc J.~Calder and D.~Slep\v{c}ev}, {\em {Properly-weighted graph {L}aplacian
  for semi-supervised learning}}, Applied Mathematics and Optimization: Special
  Issue on Optimization in Data Science, 82 (2020), pp.~1111--1159,
  \url{https://doi.org/10.1007/s00245-019-09637-3}.

\bibitem{cloninger_cautious_2021}
{\sc A.~Cloninger and H.~N. Mhaskar}, {\em Cautious active clustering}, Applied
  and Computational Harmonic Analysis, 54 (2021), pp.~44--74,
  \url{https://doi.org/10.1016/j.acha.2021.02.002}.

\bibitem{cohen2017emnist}
{\sc G.~Cohen, S.~Afshar, J.~Tapson, and A.~van Schaik}, {\em {EMNIST}: {An}
  extension of {MNIST} to handwritten letters}, 2017,
  \url{http://arxiv.org/abs/1702.05373}.

\bibitem{dasarathy_s2_2015}
{\sc G.~Dasarathy, R.~Nowak, and X.~Zhu}, {\em S2: An efficient graph based
  active learning algorithm with application to nonparametric classification},
  in Proceedings of The 28th Conference on Learning Theory, P.~Grünwald,
  E.~Hazan, and S.~Kale, eds., vol.~40 of Proceedings of Machine Learning
  Research, Paris, France, 2015, pp.~503--522.

\bibitem{dasgupta_coarse_2006}
{\sc S.~Dasgupta}, {\em Coarse sample complexity bounds for active learning},
  in Advances in {Neural} {Information} {Processing} {Systems}, vol.~18, MIT
  Press, 2006, pp.~235--242.

\bibitem{dasgupta_two_2011}
{\sc S.~Dasgupta}, {\em Two faces of active learning}, Theoretical Computer
  Science, 412 (2011), pp.~1767--1781,
  \url{https://doi.org/10.1016/j.tcs.2010.12.054},
  \url{https://doi.org/10.1016/j.tcs.2010.12.054}.

\bibitem{dasgupta_hierarchical_2008}
{\sc S.~Dasgupta and D.~Hsu}, {\em Hierarchical sampling for active learning},
  in Proceedings of the 25th International Conference on Machine Learning,
  Helsinki, Finland, July 2008, Association for Computing Machinery,
  pp.~208--215, \url{https://doi.org/10.1145/1390156.1390183}.

\bibitem{uci}
{\sc D.~Dua and C.~Graff}, {\em {UCI} machine learning repository}, 2017,
  \url{http://archive.ics.uci.edu/ml}.

\bibitem{dunlop2020large}
{\sc M.~M. Dunlop, D.~Slep{\v{c}}ev, A.~M. Stuart, and M.~Thorpe}, {\em Large
  data and zero noise limits of graph-based semi-supervised learning
  algorithms}, Applied and Computational Harmonic Analysis, 49 (2020),
  pp.~655--697.

\bibitem{flores2022analysis}
{\sc M.~Flores, J.~Calder, and G.~Lerman}, {\em Analysis and algorithms for
  $\ell_p$-based semi-supervised learning on graphs}, Applied and Computational
  Harmonic Analysis, 60 (2022), pp.~77--122.

\bibitem{gal_deep_2017}
{\sc Y.~Gal, R.~Islam, and Z.~Ghahramani}, {\em Deep {Bayesian} active learning
  with image data}, in Proceedings of the 34th International Conference on
  Machine Learning, Sydney, NSW, Australia, 2017, Journal of Machine Learning
  Research, pp.~1183--1192.

\bibitem{ijcai2018p296}
{\sc L.~Gao, H.~Yang, C.~Zhou, J.~Wu, S.~Pan, and Y.~Hu}, {\em Active
  discriminative network representation learning}, in Proceedings of the
  Twenty-Seventh International Joint Conference on Artificial Intelligence,
  {IJCAI-18}, International Joint Conferences on Artificial Intelligence
  Organization, 2018, pp.~2142--2148,
  \url{https://doi.org/10.24963/ijcai.2018/296}.

\bibitem{hanneke_bound_2007}
{\sc S.~Hanneke}, {\em A bound on the label complexity of agnostic active
  learning}, in Proceedings of the 24th International Conference on {Machine}
  Learning, New York, NY, USA, 2007, Association for Computing Machinery,
  pp.~353--360, \url{https://doi.org/10.1145/1273496.1273541}.

\bibitem{hanneke_theory_2014}
{\sc S.~Hanneke}, {\em Theory of disagreement-based active learning},
  Foundations and Trends{\textregistered} in Machine Learning, 7 (2014),
  pp.~131--309, \url{https://doi.org/10.1561/2200000037}.

\bibitem{hanneke_minimax_2015}
{\sc S.~Hanneke and L.~Yang}, {\em Minimax analysis of active learning},
  Journal of Machine Learning Research, 16 (2015), pp.~3487--3602,
  \url{http://jmlr.org/papers/v16/hanneke15a.html}.

\bibitem{he2022attribute}
{\sc Y.~He, W.~Liang, D.~Zhao, H.-Y. Zhou, W.~Ge, Y.~Yu, and W.~Zhang}, {\em
  Attribute surrogates learning and spectral tokens pooling in transformers for
  few-shot learning}, in 2022 IEEE/CVF Conference on Computer Vision and
  Pattern Recognition (CVPR), 2022, pp.~9109--9119,
  \url{https://doi.org/10.1109/CVPR52688.2022.00891}.

\bibitem{hu2020policy}
{\sc S.~Hu, Z.~Xiong, M.~Qu, X.~Yuan, M.-A. C\^{o}t\'{e}, Z.~Liu, and J.~Tang},
  {\em Graph policy network for transferable active learning on graphs}, in
  Advances in Neural Information Processing Systems, H.~Larochelle, M.~Ranzato,
  R.~Hadsell, M.~Balcan, and H.~Lin, eds., vol.~33, Curran Associates, Inc.,
  2020, pp.~10174--10185.

\bibitem{ji_variance_2012}
{\sc M.~Ji and J.~Han}, {\em A variance minimization criterion to active
  learning on graphs}, in Artificial Intelligence and Statistics, Mar. 2012,
  pp.~556--564.

\bibitem{jiang_minimum-margin_2019}
{\sc H.~Jiang and M.~Gupta}, {\em Minimum-margin active learning}, preprint
  arXiv,  (2019), \url{https://arxiv.org/abs/1906.00025}.

\bibitem{jun_graph-based_2016}
{\sc K.-S. Jun and R.~Nowak}, {\em Graph-based active learning: {A} new look at
  expected error minimization}, in 2016 {IEEE} {Global} {Conference} on
  {Signal} and {Information} {Processing} ({GlobalSIP}), Dec. 2016,
  pp.~1325--1329, \url{https://doi.org/10.1109/GlobalSIP.2016.7906056}.

\bibitem{karzand_maximin_2020}
{\sc M.~Karzand and R.~D. Nowak}, {\em Maximin active learning in
  overparameterized model classes}, IEEE Journal on Selected Areas in
  Information Theory, 1 (2020), pp.~167--177,
  \url{https://doi.org/10.1109/JSAIT.2020.2991518}.

\bibitem{kingma2013auto}
{\sc D.~P. Kingma and M.~Welling}, {\em Auto-encoding variational bayes},
  preprint arXiv,  (2013), \url{https://arxiv.org/abs/1312.6114}.

\bibitem{kingma_introduction_2019}
{\sc D.~P. Kingma and M.~Welling}, {\em An introduction to variational
  autoencoders}, Foundations and Trends® in Machine Learning, 12 (2019),
  pp.~307--392, \url{https://doi.org/10.1561/2200000056}.
\newblock Publisher: Now Publishers, Inc.

\bibitem{kushnir_diffusion-based_2020}
{\sc D.~Kushnir and L.~Venturi}, {\em Diffusion-based deep active learning},
  preprint arXiv,  (2020), \url{https://arxiv.org/abs/2003.10339}.

\bibitem{lecun-mnisthandwrittendigit-2010}
{\sc Y.~LeCun and C.~Cortes}, {\em {MNIST} handwritten digit database},
  (2010), \url{http://yann.lecun.com/exdb/mnist/}.

\bibitem{ma_sigma_2013}
{\sc Y.~Ma, R.~Garnett, and J.~Schneider}, {\em {$\Sigma$}-optimality for
  active learning on gaussian random fields}, in Advances in {Neural}
  {Information} {Processing} {Systems} 26, C.~J.~C. Burges, L.~Bottou,
  M.~Welling, Z.~Ghahramani, and K.~Q. Weinberger, eds., Curran Associates,
  Inc., 2013, pp.~2751--2759.

\bibitem{miller_model-change_2021}
{\sc K.~Miller and A.~L. Bertozzi}, {\em Model-change active learning in
  graph-based semi-supervised learning}, preprint arXiv,  (2021),
  \url{http://arxiv.org/abs/2110.07739}.

\bibitem{miller_efficient_2020}
{\sc K.~Miller, H.~Li, and A.~L. Bertozzi}, {\em Efficient graph-based active
  learning with probit likelihood via {Gaussian} approximations}, in ICML
  Workshop on Experimental Design and Active Learning, July 2020,
  \url{http://arxiv.org/abs/2007.11126}.

\bibitem{mirzasoleiman2017big}
{\sc B.~Mirzasoleiman}, {\em Big Data Summarization Using Submodular
  Functions}, PhD thesis, ETH Zurich, 2017.

\bibitem{murphy_unsupervised_2019}
{\sc J.~M. Murphy and M.~Maggioni}, {\em Unsupervised clustering and active
  learning of hyperspectral images with nonlinear diffusion}, IEEE Transactions
  on Geoscience and Remote Sensing, 57 (2019), pp.~1829--1845,
  \url{https://doi.org/10.1109/TGRS.2018.2869723}.

\bibitem{nadler2009infiniteunlabelled}
{\sc B.~Nadler, N.~Srebro, and X.~Zhou}, {\em Semi-supervised learning with the
  graph {Laplacian}: {The} limit of infinite unlabelled data}, in Proceedings
  of the 22nd International Conference on Neural Information Processing
  Systems, NIPS'09, Red Hook, NY, USA, 2009, Curran Associates Inc.,
  p.~1330–1338.

\bibitem{qiao_uncertainty_2019}
{\sc Y.-L. Qiao, C.~X. Shi, C.~Wang, H.~Li, M.~Haberland, X.~Luo, A.~M. Stuart,
  and A.~L. Bertozzi}, {\em Uncertainty quantification for semi-supervised
  multi-class classification in image processing and ego-motion analysis of
  body-worn videos}, Image Processing: Algorithms and Systems,  (2019),
  \url{https://doi.org/10.2352/issn.2470-1173.2019.11.ipas-264}.

\bibitem{sellars2022Laplacenet}
{\sc P.~Sellars, A.~I. Aviles-Rivero, and C.-B. Schönlieb}, {\em Laplacenet: A
  hybrid graph-energy neural network for deep semisupervised classification},
  IEEE Transactions on Neural Networks and Learning Systems,  (2022),
  pp.~1--13, \url{https://doi.org/10.1109/TNNLS.2022.3203315}.

\bibitem{sener_active_2018}
{\sc O.~Sener and S.~Savarese}, {\em Active learning for convolutional neural
  networks: A core-set approach}, preprint arXiv,  (2018),
  \url{http://arxiv.org/abs/1708.00489}.

\bibitem{settles_active_2012}
{\sc B.~Settles}, {\em Active Learning}, vol.~6, Morgan {\&} Claypool
  Publishers {LLC}, June 2012,
  \url{https://doi.org/10.2200/s00429ed1v01y201207aim018}.

\bibitem{shi2017weighted}
{\sc Z.~Shi, S.~Osher, and W.~Zhu}, {\em Weighted nonlocal {Laplacian} on
  interpolation from sparse data}, Journal of Scientific Computing, 73 (2017),
  pp.~1164--1177.

\bibitem{shui_deep_2020}
{\sc C.~Shui, F.~Zhou, C.~Gagné, and B.~Wang}, {\em Deep {Active} {Learning}:
  {Unified} and {Principled} {Method} for {Query} and {Training}}, preprint
  arXiv,  (2020), \url{http://arxiv.org/abs/1911.09162}.

\bibitem{simeoni_rethinking_2021}
{\sc O.~Sim\'{e}oni, M.~Budnik, Y.~Avrithis, and G.~Gravier}, {\em Rethinking
  deep active learning: {Using} unlabeled data at model training}, in The 25th
  International Conference on Pattern Recognition (ICPR), 2021,
  \url{https://doi.org/10.1109/ICPR48806.2021.9412716}.

\bibitem{slepcev2019analysis}
{\sc D.~Slepcev and M.~Thorpe}, {\em Analysis of p-{Laplacian} regularization
  in semisupervised learning}, SIAM Journal on Mathematical Analysis, 51
  (2019), pp.~2085--2120.

\bibitem{sohn2020fixmatch}
{\sc K.~Sohn, D.~Berthelot, N.~Carlini, Z.~Zhang, H.~Zhang, C.~A. Raffel, E.~D.
  Cubuk, A.~Kurakin, and C.-L. Li}, {\em Fix{M}atch: {Simplifying}
  semi-supervised learning with consistency and confidence}, in Advances in
  Neural Information Processing Systems, H.~Larochelle, M.~Ranzato, R.~Hadsell,
  M.~Balcan, and H.~Lin, eds., vol.~33, Curran Associates, Inc., 2020,
  pp.~596--608.

\bibitem{tong_support_2001}
{\sc S.~Tong and D.~Koller}, {\em Support vector machine active learning with
  applications to text classification}, Journal of Machine Learning Research, 2
  (2001), pp.~45--66.

\bibitem{vahidian_coresets_2020}
{\sc S.~Vahidian, B.~Mirzasoleiman, and A.~Cloninger}, {\em Coresets for
  estimating means and mean square error with limited greedy samples}, in
  Proceedings of the 36th {Conference} on {Uncertainty} in {Artificial}
  {Intelligence} ({UAI}), Proceedings of Machine Learning Research, Aug. 2020,
  pp.~350--359.

\bibitem{welling2016semi}
{\sc M.~Welling and T.~N. Kipf}, {\em Semi-supervised classification with graph
  convolutional networks}, in J. International Conference on Learning
  Representations (ICLR 2017), 2016.

\bibitem{xiao2017fashionmnist}
{\sc H.~Xiao, K.~Rasul, and R.~Vollgraf}, {\em Fashion-mnist: a novel image
  dataset for benchmarking machine learning algorithms}, 2017,
  \url{http://arxiv.org/abs/1708.07747}.

\bibitem{zhang2022differentiable}
{\sc N.~Zhang, L.~Li, X.~Chen, S.~Deng, Z.~Bi, C.~Tan, F.~Huang, and H.~Chen},
  {\em Differentiable prompt makes pre-trained language models better few-shot
  learners}, in International Conference on Learning Representations, 2022,
  \url{https://openreview.net/forum?id=ek9a0qIafW}.

\bibitem{Zhang_Tong_Xia_Zhu_Chi_Ying_2022}
{\sc Y.~Zhang, H.~Tong, Y.~Xia, Y.~Zhu, Y.~Chi, and L.~Ying}, {\em Batch active
  learning with graph neural networks via multi-agent deep reinforcement
  learning}, vol.~36, Jun. 2022, pp.~9118--9126,
  \url{https://doi.org/10.1609/aaai.v36i8.20897}.

\bibitem{zheng2022Simmatch}
{\sc M.~Zheng, S.~You, L.~Huang, F.~Wang, C.~Qian, and C.~Xu}, {\em Sim{M}atch:
  {Semi}-supervised learning with similarity matching}, in 2022 IEEE/CVF
  Conference on Computer Vision and Pattern Recognition (CVPR), 06 2022,
  pp.~14451--14461, \url{https://doi.org/10.1109/CVPR52688.2022.01407}.

\bibitem{zhou2018graph}
{\sc J.~Zhou, G.~Cui, S.~Hu, Z.~Zhang, C.~Yang, Z.~Liu, L.~Wang, C.~Li, and
  M.~Sun}, {\em Graph neural networks: {A} review of methods and applications},
  AI Open, 1 (2020), pp.~57--81,
  \url{https://doi.org/https://doi.org/10.1016/j.aiopen.2021.01.001}.

\bibitem{zhu_robust_2019}
{\sc D.~Zhu, Z.~Li, X.~Wang, B.~Gong, and T.~Yang}, {\em A robust zero-sum game
  framework for pool-based active learning}, in The 22nd {International}
  {Conference} on {Artificial} {Intelligence} and {Statistics}, Apr. 2019,
  pp.~517--526.

\bibitem{zhu_semi-supervised_2003}
{\sc X.~Zhu, Z.~Ghahramani, and J.~Lafferty}, {\em Semi-supervised learning
  using {Gaussian} fields and harmonic functions}, in Proceedings of the 20th
  {International} {Conference} on {International} {Conference} on {Machine}
  {Learning}, Washington, DC, USA, Aug. 2003, AAAI Press, pp.~912--919.

\bibitem{zhu_combining_2003}
{\sc X.~Zhu, J.~Lafferty, and Z.~Ghahramani}, {\em Combining active learning
  and semi-supervised learning using {Gaussian} fields and harmonic functions},
  in International Conference on Machine Learning (ICML) 2003 workshop on The
  Continuum from Labeled to Unlabeled Data in Machine Learning and Data Mining,
  2003, pp.~58--65.

\end{thebibliography}

\pagebreak
\begin{center}
	\textbf{ {\large SUPPLEMENTAL MATERIAL: Poisson reweighted Laplacian sampling for graph-based active learning}}
\end{center}

\section{Imbalanced dataset results} \label{smsec:imbalanced-results}

We provide the results of our minimum norm active learning method for a few datasets with significant cluster and class imbalances. For the \textbf{MNIST} and \textbf{FASHIONMNIST} datasets, we subsample the clusters (i.e. digits) to obtain disparate numbers of points in each cluster. The classification structure still follows the procedure performed in Section \ref{sec:larger-datasets}, where digit $k$ is assigned to class $k \operatorname{mod} 3$. We refer to the resulting datasets as \textbf{MNIST-IMB} and \textbf{FASHIONMNIST-IMB}. 

We create the \textbf{EMNIST-VCD} dataset by classifying each digit or letter in \textbf{EMNIST} dataset according to the classes: \textit{vowel}, \textit{consonant}, and \textit{digit} (hence the suffix ``VCD''). This imposes a severe class imbalance, as there are only 10 digits and 8 vowel letters compared to 29 consonants. Recall that the \textbf{EMNIST} dataset contains both uppercase and lowercase letters when the corresponding shape of the letters are distinct between the two cases. 

\begin{figure}[!h]
    \centering
    \subfigure[Accuracy]{\includegraphics[width=0.49\textwidth]{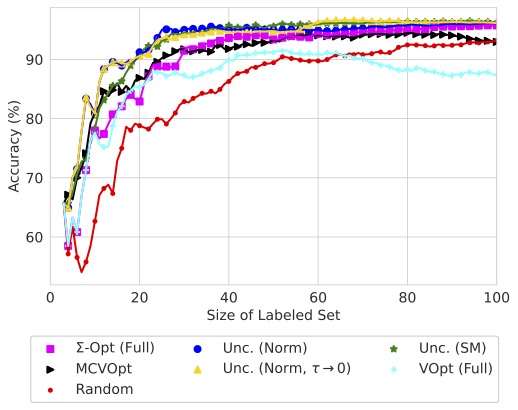}}
    \subfigure[Cluster Proportion]{\includegraphics[width=0.49\textwidth]{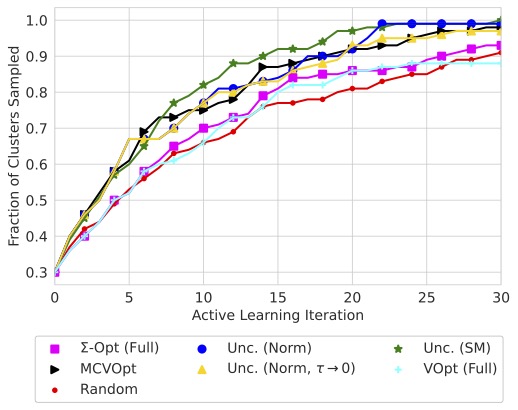}}
    \caption{Accuracy Results (a) and Cluster Proportion (b) plots for \textbf{MNIST-IMB} dataset. }
    \label{fig:mnistimb-results}
\end{figure}
\begin{figure}[!h]
    \centering
    \subfigure[Accuracy]{\includegraphics[width=0.49\textwidth]{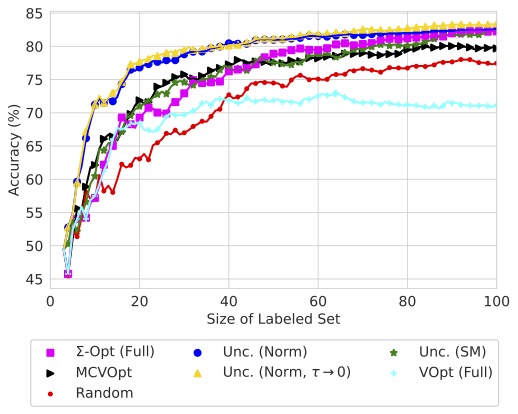}}
    \subfigure[Cluster Proportion]{\includegraphics[width=0.49\textwidth]{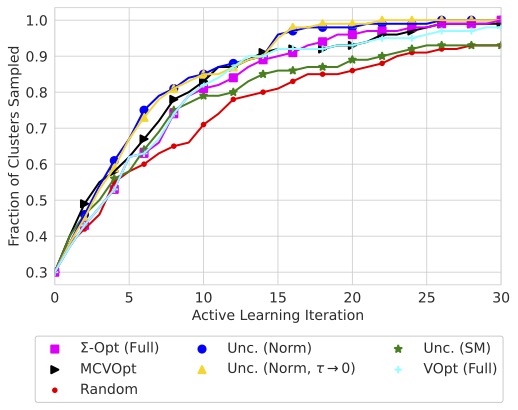}}
    \caption{Accuracy Results (a) and Cluster Proportion (b) plots for \textbf{FASHIONMNIST-IMB} dataset. }
    \label{fig:fashionmnistimb-results}
\end{figure}
\begin{figure}[!h]
    \centering
    \subfigure[Accuracy]{\includegraphics[width=0.49\textwidth]{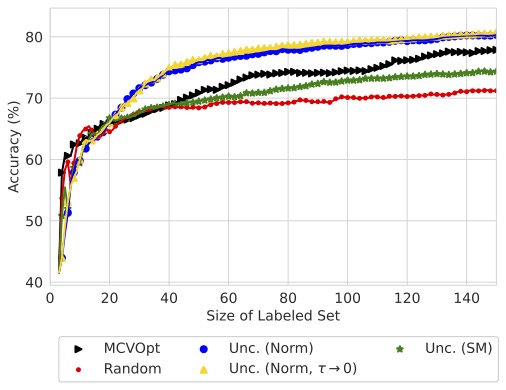}}
    \subfigure[Cluster Proportion]{\includegraphics[width=0.49\textwidth]{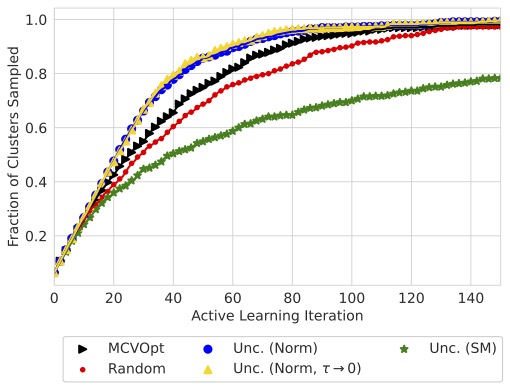}}
    \caption{Accuracy Results (a) and Cluster Proportion (b) plots for \textbf{EMNIST-VCD} dataset. }
    \label{fig:emnistvcd-results}
\end{figure}

\subsection{Discussion of results}

Similar to what we observed in Section \ref{sec:larger-datasets}, our acquisition functions Unc.~(Norm) and Unc.~(Norm, $\tau \rightarrow 0$) perform favorably across all tests, for both accuracy and cluster exploration. Recall that the VOpt (Full) and $\Sigma$-Opt (Full) methods are performing the ``full'' corresponding calculations on only a subset of 500 randomly selected points from the current set of unlabeled data at each iteration to avoid a full matrix inversion of the corresponding graph Laplacian matrix. It is interesting that we observe--contrary to what we observe in the toy experiments--that decreasing $\tau$ throughout these experiments does not necessarily improve the performance of the minimum norm acquisition function. An interesting direction of future work would be to investigate why this decay schedule for $\tau$ does not seem to be useful in these higher-dimensional datasets.

\section{Proofs from continuum analysis}

Here, we give the proofs of the main results from Section \ref{sec:theory}. We also include some preliminary results that are useful in the proofs of the main results.

\subsection{Proofs from Section \ref{sec:1d-theory}}\label{sec:eqs-different-regions}

\begin{figure}[h]
    \centering
    \includegraphics[width=0.9\textwidth]{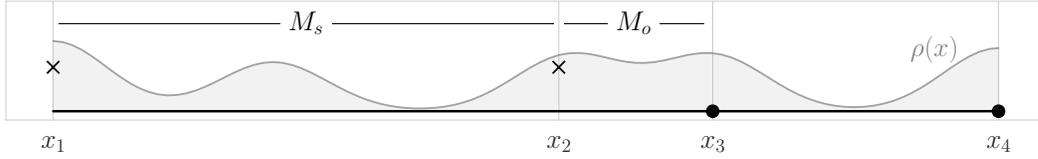}
    \vspace{-1em}
    \caption{Visualization of 1D continuum example setup. The density $\rho(x)$ is plotted in gray, while the labeled points $x_1, x_2, x_3, x_4$ are plotted where the corresponding label is denoted by $\times$ or a solid dot. $\R_s$ marks the length between two similarly labeled points, while $\R_o$ marks the length between two oppositely labeled points.}
    \label{fig:1d-init-sm}
\end{figure}

Without loss of generality, we may consider solving a given differential equation with boundary conditions from above on an interval $(0, \R)$, where $\R = x_{i+1} - x_i$ and the density on said interval is translated accordingly. As a result, we will be able to compare the value of $\mcl A$ for each possible of the corresponding problems \eqref{eq:de-opp-labels} - \eqref{eq:de-same-labels} as a function of the interval length $R$. 

The minimum norm acquisition function (Unc.~(Norm)) in this 1D binary example will select 
\[
    x^\ast = \argmax_{x \in \Omega} \ \mcl A (x) = \argmax_{x \in \Omega} \ \sqrt{u_0^2(x) + u_1^2(x)},
\]
where $u_0, u_1$  are solutions to ordinary differential equations (ODE) with appropriate boundary conditions that reflect the classification of the labeled points to identify query points.

Suppose that the $i^{th}$ interval contains \textit{oppositely labeled endpoints}, then considering the translation of the domain to $(0,\R_o)$ and then $u_0, u_1$ would respectively solve the differential equations:
\begin{equation} \label{eq:de-opp-labels}
    \begin{cases}
        -\rho^{-1}(\rho u'_0)' + \tau u_0 = 0,  & x \in (0, \R_o) \\
        u_0(0) = 1 & \\
        u_0(\R_o) = 0&\\
    \end{cases} \hspace{3em} \begin{cases}
        -\rho^{-1}(\rho u'_1)' + \tau u_1 = 0,  & x \in (0,\R_o) \\
        u_1(0) = 0 & \\
        u_1(\R_o) = 1&\\
    \end{cases}.
\end{equation}
For this type of region (between oppositely labeled points) we will denote the acquisition function as $\mcl A_o$ to emphasize its dependence on the region type, and we write
\[
    \mcl A_o(x) = \sqrt{u_0^2(x) + u_1^2(x)}.
\]

Likewise, suppose that the $i^{th}$ interval contains \textit{similarly labeled endpoints}. Without loss of generality, we may assume that $y_i = y_{i+1} = 0$ and considering the translation of the domain to $(0,\R_s)$, let $v_0, v_1$ respectively satisfy the differential equations: 
\begin{equation} \label{eq:de-same-labels}
    \begin{cases}
        -\rho^{-1}(\rho v'_0)' + \tau v_0 = 0  & x \in (0,\R_s) \\
        v_0(0) = 1 & \\
        v_0(\R_s) = 1&\\
    \end{cases} \hspace{3em} \begin{cases}
        -\rho^{-1}(\rho v'_1)' + \tau v_1 = 0  & x \in (0,\R_s) \\
        v_1(0) = 0 & \\
        v_1(\R_s) = 0&\\
    \end{cases}.
\end{equation}
Then we have $v_1 \equiv 0$ and the acquisition function in this type of region will be denoted $\mcl A_s$ and take the form
\[
    \mcl A_s(x) = \sqrt{v_0^2(x) + v_1^2(x)} = v_0(x).
\]

In both of the above problems, the hyperparameter $\tau$ is assumed to be non-negative. It is useful to note that, for each problem, the solutions in the case that $\tau=0$ constitute super-solutions to the $\tau > 0$ case for the same problem. This is a consequence of a comparison principle argument in each problem using the boundary conditions. We show the argument for the oppositely labeled point problem \eqref{eq:de-opp-labels} and similar arguments hold for the other variants on the boundary conditions. 
\begin{lemma} \label{lemma:comp-principle}
    Let $u_\tau, u$ be the solutions to \eqref{eq:de-opp-labels} for $\tau > 0$ and $\tau = 0$, respectively. Then, $u \ge u_\tau$ on the interval under consideration. 
\end{lemma}
\begin{proof}
    Due to a maximum principle for $u$, it is clear to see that since $u$ satisfies \eqref{eq:de-opp-labels} for $\tau = 0$ we have
    \[
        -\frac{1}{\rho} (\rho^2 u')' + \tau u = \tau u \ge 0.
    \]
    Then $v = u - u_\tau$ satisfies the differential equation 
    \begin{equation*}
        \begin{cases}
            -\frac{1}{\rho}(\rho v')' + \tau v \ge 0  & x \in (0,\R) \\
            v(0) = 0 & \\
            v(\R) = 0 &\\
        \end{cases}
    \end{equation*}
    and we must have then that $0 \le v = u - u_\tau$. We conclude that $u \ge u_\tau$.
\end{proof}

\subsubsection{Warmup: Constant density case} \label{smsec:warmup-const-density}

In the case that $\rho(x) \equiv \rho > 0$ is constant over each of the intervals of interest, we can explicitly compute the solutions to each of the above problems. Letting $a := \sqrt{\frac{\tau}{\rho}}$ in the case that $\tau > 0$, then 
\begin{align*}
    u_0(x) &= \frac{e^{a(\R_o - x)} - e^{-a (\R_o-x)}}{e^{a \R_o} - e^{-a \R_o}}, \quad u_1(x) = \frac{e^{a x} - e^{-a x}}{e^{a \R_o} - e^{-a \R_o}} \\
    v_0(x) &= \frac{e^{a(\frac{\R_s}{2} - x)} + e^{-a (\frac{\R_s}{2}-x)}}{e^{a \frac{\R_s}{2}} + e^{-a \frac{\R_s}{2}}}, \quad v_1(x) \equiv 0.
\end{align*}
It is straightforward to check that $x^\ast_o = \frac{\R_o}{2}$ is the unique minimizer of $\mcl A_o$ and likewise that $x^\ast_s = \frac{\R_s}{2}$ for $\mcl A_o$. This is due to the inherent symmetry in the solutions $u_0, u_1, v_0,$ and $v_1$. Thus, we can write the acquisition functions values for the oppositely labeled problem ($\mcl A_o$) and similarly labeled problems ($\mcl A_s$) as respectively
\begin{align*}
    \min_{x \in [0,\R_o]}  \mcl A_o\lp x \rp &= \mcl A_o \lp \frac{\R_o}{2}\rp = \frac{1}{\sqrt{2} \cosh\lp \sqrt{\frac{\tau}{\rho}}\frac{\R_o}{2}\rp}, \\
    \min_{x \in [0,\R_s]} \mcl A_s \lp x \rp &= \mcl A_s \lp \frac{\R_s}{2} \rp  = \frac{1}{\cosh\lp\sqrt{\frac{\tau}{\rho}}\frac{\R_s}{2}\rp }.
\end{align*}

In the case that $\tau = 0$, note that $v_0(x) \equiv 1$ and 
\begin{equation*}
    u_0(x) = \frac{\R_o - x}{\R_o}, \quad u_1(x) = \frac{x}{\R_o},
\end{equation*}
so that we may write the respective acquisition function values as
\begin{align*}
    \min_{x \in [0,\R_o]} \mcl A_o(x) &= \mcl A_o \lp \frac{\R_o}{2} \rp  = \frac{1}{\sqrt{2}}, \\
    \min_{x \in [0,\R_s]} \mcl A_s(x) &= 1.
\end{align*}

We now state two propositions that clarify the effect to which $\tau > 0$ affects choosing an ``explorative'' vs ``exploitative'' point. We make the distinction here that an exploitative query point would lie between two oppositely labeled points \emph{that are relatively close together} (i.e.,\ $\R_o$ is small compared to $\R_s$). In the case that $\R_o \approx \R_s$ or $\R_o \ge \R_s$, then selecting a query point between oppositely labeled points would be just as ``explorative'' as querying between similarly labeled points. This idea is quantified in the following propositions.

\begin{proposition}[Exploration requirement on $\tau > 0$]\label{prop:warmup-exploration}
    Suppose that the density is constant for the intervals of interest. If $\tau > 0$ and the oppositely and similarly labeled intervals' lengths are fixed and satisfy $\R_o = \beta \R_s$ for some $\beta \in [0, \frac{1}{\sqrt{2}}]$, then 
    \begin{equation} \label{eq:exploration-const-rho}
        \mcl A_s\lp\frac{\R_s}{2}\rp < \mcl A_o\lp\frac{\R_o}{2}\rp,
    \end{equation}
    as long as $\sqrt{\tau} \R_s > 2\sqrt{\rho} t^\ast$, where $t^\ast > 0$ is the solution to the transcendental equation
    \[
        \cosh(t^\ast) = \sqrt{2} \cosh(\beta t^\ast).
    \]
\end{proposition}
\begin{proof}
    Given the relationship $\R_o = \beta \R_s$, then satisfying the relation \eqref{eq:exploration-const-rho} is equivalent to finding a values of $\tau > 0$ such that 
    \[
        \cosh\lp\sqrt{\frac{\tau}{\rho}}\frac{\R_s}{2}\rp > \sqrt{2}\cosh\lp\sqrt{\frac{\tau}{\rho}}\frac{\beta \R_s}{2}\rp.
    \]
    We can equivalently consider the function $f(t) = \sqrt{2}\cosh(\beta t) - \cosh(t)$. 
    Since $\beta \in (0, \frac{1}{\sqrt{2}}]$ and $t \ge 0$, 
    \begin{align*}
        f'(t) &= \beta \sqrt{2} \sinh(\beta t) - \sinh(t) < 0, \\
        f''(t) &= \beta^2 \sqrt{2} \cosh(\beta t) - \cosh(t) < 0
    \end{align*}
    which allows us to conclude that there exists a unique root $t^\ast > 0$ of $f$ (i.e., $f(t^\ast) = 0$) and that $f(t) < 0$ for all $t > t^\ast$. 
    In other words, we have that \eqref{eq:exploration-const-rho} holds as long as 
    \begin{equation} \label{eq:tau-Ms-explore-relation}
        \frac{\sqrt{\tau}\R_s}{2\sqrt{\rho}} > t^\ast.
    \end{equation}
\end{proof}

\begin{proposition}[Exploitation default when $\tau$ too small]\label{prop:warmup-exploitative}
    Suppose that the density is constant for the intervals of interest. If the oppositely and similarly labeled intervals' lengths are fixed and satisfy $\R_o = \beta \R_s$ for some $\beta > 0$, then 
    \begin{equation} \label{eq:exploitation-const-rho}
        \mcl A_s\lp\frac{\R_s}{2}\rp > \mcl A_o\lp\frac{\R_o}{2}\rp,
    \end{equation}
    as long as  \[
        0 \le \tau \le \frac{2\rho}{\R_s^2}.
    \]
\end{proposition}
\begin{proof}
    When $\tau = 0$, it is trivial to see that \eqref{eq:exploitation-const-rho} holds, since $\mcl A_s(\frac{\R_s}{2}) = \mcl A_s(x) = 1 > \mcl A_o(\frac{\R_o}{2})$ always holds for all $x$ since $v_0(x) \equiv 1$ and $v_1(x) \equiv 0$. We therefore restrict to the case $\tau > 0$ for the remainder of the proof.
    
    Similar to the proof of Proposition \ref{prop:warmup-exploration}, we consider the function $f(t) = \sqrt{2}\cosh(\beta t) - \cosh(t)$. It suffices show that $f(t) > 0$ for all $t \ge 0$. 
    Note that $f$ is an increasing function of the parameter $\beta > 0$ which implies the lower bound $f(t) > \sqrt{2} - \cosh(t)$. Restricting $t \in [0, \frac{1}{\sqrt{2}}]$, we see then that 
    \[
        \min_{t\in [0, \frac{1}{2}]}\ f(t) > \sqrt{2} - \cosh\lp \frac{1}{\sqrt{2}}\rp > 0.
    \]
    Thus, we have that $f(t) > 0$ for all $0 \le t \le \frac{1}{\sqrt{2}}$ and we conclude that \eqref{eq:exploitation-const-rho} holds for all $\tau \ge 0$ that satisfy
    \[
        \tau \le \frac{2 \rho}{\R_s^2}.
    \]
\end{proof}

The result of Propositions \ref{prop:warmup-exploration} and \ref{prop:warmup-exploitative} is that if $\tau \ge 0$ is chosen to be \emph{too small} compared to the length of the interval, $\R_s$, then no matter the relative size of $\R_o$, the oppositely labeled problem's region will be queried rather than the similar labeled problem's region. On the other hand, as long as $\tau > 0$ is \emph{large enough} (as quantified in \eqref{eq:tau-Ms-explore-relation}), then explorative query points will be selected in the similarly labeled regions. We now generalize this situation to the case when the density $\rho(x)$ is no longer assumed to be constant. 

\begin{figure}[t]
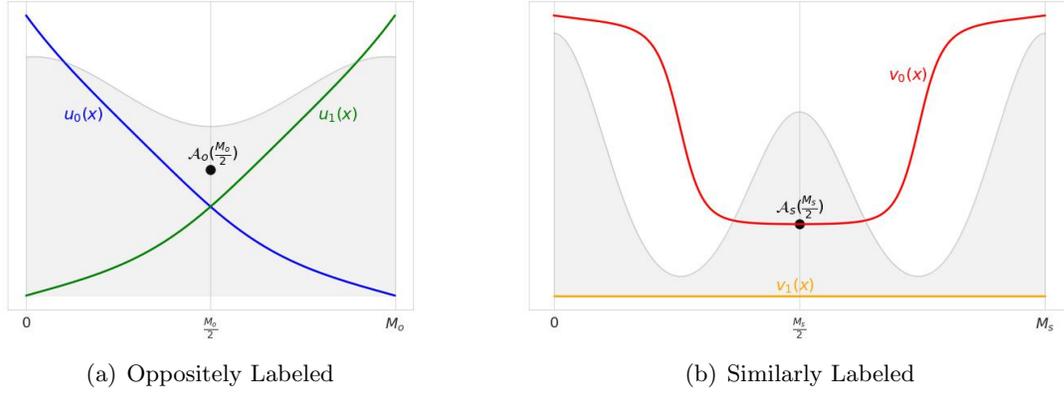

    \centering
    \subfigure[Oppositely Labeled]{\includegraphics[height=12em]{imgs/1d_opp_plot2.jpg}} 
    \hspace{3em}
    \subfigure[Similarly Labeled]{\includegraphics[height=12em]{imgs/1d_same_plot.jpg}} 
    \caption{Visualization of one-dimensional solutions to \eqref{eq:rw_lap_continuum} in oppositely (a) and similarly (b) labeled regions. Solutions $u_0,u_1$ (blue, green lines) to \eqref{eq:rw_lap_continuum} are shown in panel (a) when endpoints have opposite labels, while $v_0,v_1$ (red, orange lines) to \eqref{eq:rw_lap_continuum} are shown in panel (b) when endpoints have the same labels. The background density $\rho$ in each respective region is shown in gray, and the acquisition function value at the midpoint of the interval is shown as a black dot. The minimum acquisition function value occurs at the midpoint if the density is symmetric, which we show here for simpler presentation.  }
    \label{fig:solutions-acqfuncs-sm}
\end{figure}

\subsubsection{Symmetric density generalization} 

We now turn to a generalization of the case of constant density $\rho(x) \equiv \rho$ we addressed in the previous section; namely, we now assume that the density in both types of regions satisfy $\rho(x) \in [\rho_{min}, \rho_{max}]$. 
We introduce an additional assumption on the density $\rho$ in the similarly labeled problem that allows us to explicitly determine that the minimizers of $\mcl A$ on the interval $(0,\R_s)$ will occur at the bisection point, $x^\ast_s = \R_s/2$. The propositions and arguments of Subsection
\ref{subsec:sim-problem-R/2} rigorously establish this claim, after which we determine bounds on the corresponding acquisition function values so that we can compare the minimizers in each respective region. Our argument for bounding the acquisition function value between oppositely labeled points, however, does \textit{not} require this assumed symmetry. 

We now state the symmetry assumption for the density $\rho(x)$ on a given interval $(0, \R)$:
\begin{assumption} \label{assumption:symmetry}
    Assume that the density $0 < \rho_{min} \le \rho(x) \le \rho_{max} < +\infty$ is \textit{symmetric about the midpoint} of the interval, $\R/2$. That is,  
    \[ 
        \rho(x) = \rho(\R - x)
    \]
    for $x \in [0, \R]$.
\end{assumption}

\subsubsection{Minimizer of similarly labeled problem} \label{subsec:sim-problem-R/2}

As mentioned previously in Section \ref{sec:eqs-different-regions}, the boundary conditions for the similarly labeled problem implies that $\mcl A_s(x) = v_0(x)$. Furthermore, the symmetry of the density $\rho$ (Assumption \ref{assumption:symmetry}) now implies that the function $\tilde{v}(x) := v_0(\R_s-x)$ also satisfies \eqref{eq:de-same-labels}, and so we can conclude the symmetry of the solution $v_0(x) = v_0(\R_s - x)$.
We begin with a brief lemma to show that the solution $v_0 \ge 0$.
\begin{lemma} \label{lemma:sim-label}
    Given that the density $\rho$ satisfies Assumption \ref{assumption:symmetry}, then the solution $v_0$ satisfies
    \[
        v_0(x) \in (0,1]
    \]
    for $x \in [0,\R_s]$ and $\tau \ge 0$. 
\end{lemma}
\begin{proof}
    From a simple application of the strong maximum (minimum) principle for this elliptic equation, we see that if $x_0 \in (0, \R_s)$ is an assumed minimizer of $v_0$, then if $v_0(x_0) \le 0$ we would have that $v_0$ would necessarily be constant. This contradicts the boundary conditions that $v_0(0) = v_0(\R_s) = 1 > 0$ and the assumed regularity of the solution $v_0$. Thus, we conclude that $v_0(x) > 0$. Likewise, the strong maximum principle implies that $v_0(x) \le \max\{ v_0(0), v_0(\R_s)\} = 1$, so we conclude that $v_0(x) \in (0, 1]$.
\end{proof}

We now state the following proposition which shows that the minimizer of $\mcl A_s$ for the similarly labeled problem \eqref{eq:de-same-labels} occurs at $x^\ast = \frac{\R_s}{2}$.
\begin{proposition} \label{prop:sim-label-problem}
    Given that the density $\rho$ satisfies Assumption \ref{assumption:symmetry}, then the acquisition function $\mcl A_s$ is minimized at the bisecting point of the interval, $x^\ast = \R_s/2$, when $\tau > 0$. When $\tau = 0$, then $v \equiv 1$, and $\mcl A_s(x) = 1$ for all $x \in [0, \R_s]$.
\end{proposition}
\begin{proof}
    Since the value for $\mcl A_s$ in this case reduces to $\mcl A_s(x) = v(x)$, then a critical point $x^\ast$ of $\mcl A$ must satisfy
    \[
        v'(x^\ast) = 0.
    \]
    Due to the symmetry $v(x) = v(\R_s - x)$, it is straightforward to see then that $v'\lp\frac{\R_s}{2}\rp = 0$ and so we have that $\frac{\R_s}{2}$ is a critical point. In the $\tau = 0$ case, we see that $v \equiv 1$ by integrating the differential equation 
    \[
        (\rho^2 v')' = 0
    \]
    and by plugging in the boundary conditions $v(0) = 1 = v(\R_s)$. For the rest of the proof, we assume $\tau >0$.
    
    The derivative of the solution $v$ satisfies
    \[
        \rho^2(x) v'(x) = \rho^2(0) v'(0) + \int_0^x \tau \rho v dz,
    \]
    and so a critical point must satisfy the equation
    \begin{equation}\label{eq:v'-relation}
        0 = \rho^2(0) v'(0) + \int_0^{x^\ast} \tau \rho v dz.
    \end{equation}
    Viewed as a function of $x$, $\int_0^{x} \tau \rho v dz$ is strictly increasing function due to the positivity of $v(x^\ast) > 0$ (Lemma \ref{lemma:sim-label}). This means that for $\tau > 0$, $x^\ast = \frac{\R_s}{2}$ is the unique critical point of $\mcl A_s$ on $(0,\R_s)$.

    Due to \eqref{eq:v'-relation}, the regularity of $v$, and the conclusion of the above paragraph, we see that $v'(x) < 0$ for all $x \in [0, \frac{\R_s}{2})$. Thus, appealing one last time to the symmetry of $v$ about $\frac{\R_s}{2}$, we conclude that $x^\ast = \R_s/2$ is the unique \textit{minimizer} of both $v$ and $\mcl A$ on $(0, \R_s/2)$ for the similarly labeled problem \ref{eq:de-same-labels}.
\end{proof}

From this, we conclude that the minimum value in the similarly labeled problem is 
\[
    \min_{x \in [0, \R_s]} \ \mcl A_s(x) = v_0\lp \frac{\R_s}{2}\rp,
\]
as long as $\tau > 0$.
In contrast, when $\tau = 0$, the acquisition function is constant on the interval because $v_0(x) \equiv 1$, and will consequently \emph{never be queried in the active learning process}.

\subsubsection{Comparing minimizers}

We now compare the minimizing values in each type of region so that we can determine under what additional conditions we can conclude that the global minimizer will lie in a given region. This allows us to quantify exactly when \textit{explorative} points (i.e., minimizers of the \textit{similarly labeled} problem in ``wide'' regions) versus \textit{exploitative} points (i.e., minimizers of the \textit{oppositely labeled} problem in ``thin'' regions) will be chosen via this minimum norm acquisition function. 

As stated previously, when $\tau = 0$ the acquisition function $\mcl A_s \equiv 1$. On the other hand, the oppositely labeled regions will necessarily have $\mcl A_o(x) < 1$ for $x \in (0, \R_o)$, and so we conclude that \textbf{if $\tau = 0$, explorative points would never be chosen by this acquisition function}. This noteworthy observation helps to explain why uncertainty sampling--when not properly designed--can result in the selection of overly-exploitative query points during the active learning process. 

We now turn our attention to the more interesting case of when $\tau > 0$ and give conditions that would ensure the selection of explorative query points (i.e., in similarly labeled problem regions that are relatively large).
We first lower bound the value of $\mcl A_o(x)$ and then upper bound $\mcl A_s(\R_s/2)$, which will allow us to certify when the global minimizer of $\mcl A$ would occur in an explorative region (i.e., when $\mcl A_s(\R_s/2) < \mcl A_o(\R_o/2)$).


\subsubsection{Lower bound for $\mcl A_o$} \label{subsub:opp-label-bound}

To lower bound the value of $\mcl A_o$, we require a mild assumption (Assumption \ref{assumption:end-intervals}) about the density $\rho(x)$ that will allow us to lower bound the solutions $u_0, u_1$. 
\begin{assumption}\label{assumption:end-intervals}
    There exists $0 < \epsilon < \frac{\R}{2}$ such that the density $\rho(x)$ satisfies $\rho'(x) \le 0$ for all $x \in (0, \epsilon)$ and $\rho'(x) \ge 0$ for all $x \in (\R - \epsilon, \R)$.
\end{assumption}
This assumption allows us to apply the same argument to both $u_0$ and $u_1$, and accordingly we turn our attention to lower bounding the solution $u_0$. Consider the following ansatz
\[
    \phi(x) := \frac{e^{\theta(\R_o-x)} - e^{-\theta(\R_o-x)}}{e^{\theta \R_o} - e^{-\theta \R_o}} = \frac{\sinh \lp \theta (\R_o-x)\rp}{\sinh \lp \theta \R_o \rp},
\]
where we have yet to choose $\theta > 0$. Note that $\phi(0) = u_0(0) = 1, \phi(\R_o) = u_0(\R_o) = 0$, and 
\[
    \phi'(x) = \theta \frac{\cosh \lp \theta (\R_o-x)\rp}{\sinh \lp \theta \R_o \rp}, \quad \phi''(x) = \theta^2 \phi(x).
\]
Then, plugging $\phi$ into the differential operator we may bound
\begin{align*}
    \tau \phi - \frac{1}{\rho}\lp \rho^2 \phi'\rp' &= \rho \lp \tau \rho^{-1} - \phi'' - 2 \lp \log \rho \rp' \phi'\rp \\
    &= \frac{\rho}{\sinh\lp \theta \R_o \rp } \left[ \lp \tau \rho^{-1}  - \theta^2 \rp \sinh\lp \theta (\R_o - x) \rp - 2 \theta \lp \log \rho \rp' \cosh\lp \theta (\R_o - x) \rp \right].
\end{align*}
We then set $\theta^2 > \tau \rho_{o,min}^{-1}$.  Let $\epsilon > 0$ be given from Assumption \ref{assumption:end-intervals}. Then we can bound
\[
    \tau \phi - \frac{1}{\rho}\lp \rho^2 \phi'\rp' \le 0
\]
for $x \in (\R_o - \epsilon, \R_o)$. For $x \in (0, \R_o-\epsilon)$, we may bound
\begin{align*}
    \tau \phi - \frac{1}{\rho}\lp \rho^2 \phi'\rp' &\le \frac{\rho}{\sinh\lp \theta \R_o \rp } \left[ \lp \tau \rho_{o,min}^{-1}  - \theta^2 \rp \sinh\lp \theta (\R_o - x) \rp \right.\\
    & \left. \qquad \qquad + 2 \theta \lp \max_{z \le \R_o - \epsilon} \left| \lp \log \rho \rp'\right| \rp   \cosh\lp \theta (\R_o - x) \rp \right],
\end{align*}
so it suffices to show that the function 
\[
    g(x) := - a_0 \sinh\lp \theta (\R_o - x) \rp + a_1 \cosh\lp \theta (\R_o - x) \rp \le 0
\]
for all $x \in (0, \R_o -\epsilon)$, where $a_0 = \theta^2 - \tau\rho_{min}^{-1} > 0$ and $a_1 = 2 \theta \lp \max_{z \le \R_o - \epsilon} \left| \lp \log \rho \rp' \right|\rp =: 2 \theta C_\rho > 0$. The condition $g(x) \le 0$ is equivalent to 
\[
    \tanh\lp \theta (\R_o - x) \rp \ge \frac{a_1}{a_0},
\]
which implies that we must necessarily have $a_1 < a_0$, since $\tanh(x) \le 1$. Since $\tanh$ is an increasing function in its inputs, we may bound
\[
    \tanh\lp \theta (\R_o - x) \rp \ge \tanh\lp \theta \epsilon \rp \ge \frac{a_1}{a_0} = \frac{2 \theta C_\rho}{\theta^2 - \tau\rho_{o,min}^{-1}}.
\]

As we must have $\tau \rho_{o,min}^{-1} < \theta^2$, we set $\tau \rho_{o,min}^{-1} = (1 - \lambda) \theta^2$ for some $\lambda \in (0, 1)$ which allows us to rewrite the above condition as
\begin{equation} \label{eq:gamma-tanh-bound}
    \theta \tanh\lp \theta \epsilon \rp \ge \frac{2}{\lambda}C_\rho.
\end{equation}
We conclude that as long as \eqref{eq:gamma-tanh-bound} holds for $\theta, \lambda$ such that $\tau \rho_{o,min}^{-1} = (1- \lambda) \theta^2$, then we have that 
\[
    \tau \phi - \frac{1}{\rho}\lp \rho^2 \phi'\rp' \le 0 \qquad \text{ for all } x \in (0, \R_o).
\]

A straightforward application of a strong maximum principle for this elliptic differential operator allows us to conclude that $\phi < u_0$. A similar argument can be applied to the function $u_1(x)$ with ansatz $\varphi(x) = \frac{e^{\theta x} - e^{-\theta x}}{e^{\theta \R_o} - e^{-\theta \R_o}}$ and the help of Assumption \ref{assumption:end-intervals} to conclude $\varphi < u_1$. In fact, the same value of $\theta$ works for both $\phi$ and $\varphi$, which allows us to write $\varphi(x) = \phi(\R_o - x)$.  We can therefore bound the acquisition function values by
\begin{align*}
    \mcl A_o(x) &= \sqrt{ u_0^2\lp x\rp  + u_1^2\lp x\rp } \ge \sqrt{ \phi^2\lp x\rp  + \phi^2\lp \R_o - x\rp } \\
    &\ge \sqrt{2} \phi\lp \frac{\R_o}{2} \rp = \frac{1}{\sqrt{2} \cosh\lp \sqrt{\frac{\tau}{(1 -\lambda)\rho_{o,min}}}  \frac{\R_o}{2} \rp }\\
    &\ge \frac{1}{\sqrt{2}} \exp \lp -  \frac{\tau \R_o^2}{8(1 -\lambda)\rho_{o,min}}\rp, 
\end{align*}
since the function $\phi$ has the same form as the constant density case functions from Section \ref{subsec:warmup-const-density} and has minimum acquisition function value at the interval's midpoint, $\frac{\R_o}{2}$.

\subsubsection{Upper bound for $\mcl A_s\lp \frac{\R_s}{2}\rp$}

For the similarly labeled problem, we will more directly use the information of $\rho$ to compute an upper bound for $v_0\lp\frac{\R_s}{2}\rp$. For $\delta \in (\rho_{min}, \rho_{max})$, we consider the value
\begin{equation}
    \alpha(\delta) := \R_s^{-1}\int_{x: \rho(x) \le \delta} dx \in (0, 1), 
\end{equation}
which is simply the proportion of the interval $(0, \R_s)$ corresponding to the sublevel sets of $\rho(x)$. Intuitively, in order to ensure ``enough'' decay in the solution $v_0$, we will need that $\alpha(\delta)$ is large for relatively small values of $\delta$; in other words, we will need significant portions of the interval $(0, \R_s)$ to have relatively small density (i.e., $\rho(x) \le \delta)$. 

Let $\{I_i\}_{i=1}^n$ denote the sequence of subintervals in $\lp 0,\frac{\R_s}{2}\rp$ such that the following hold
\begin{align*}
    \bigcup_{i=1}^n \bar{I_i} &= \left[0, \frac{\R_s}{2}\right], \qquad 
    I_i \cap I_j = \emptyset \text{ for } i \not= j.
\end{align*}
Further, define
\begin{align*}
    \mcl I_\ell &:= \{ I_i : \rho(x) \le \delta \text{ for all } x\in I_i \}, \qquad 
    \mcl I_h := \{ I_i : \rho(x) > \delta \text{ for all } x\in I_i \}.
\end{align*} 
Assume that each $I_i$ is the largest contiguous subinterval where the final condition is met, so that if $I_i \in \mcl I_\ell$, then $\rho(x) > \delta$ for all $x \in I_{i+1}$ and $x \in I_{i-1}$ for $i-1,i+1 \in \{1, 2, \ldots, n\}$. Let $\ell_i = |I_i|$ denote the length of interval $I_i$ and $x_i$ be the midpoint of the subinterval $I_i$. Let $n_\ell, n_h$ respectively denote the number of ``low'' and ``high'' density subintervals in the sequence of $n = n_\ell + n_h$ subintervals. The following lemma demonstrates that on each interval $I_i \in \mcl I_\ell$ we can characterize the extent of the decay of the solution $v_0$ on $I_i$. 

\begin{lemma} \label{lemma:supersolution-similar-problem}
    Let $u$ satisfy the differential equation
    \begin{equation}
            \tau u - \frac{1}{\rho}\lp \rho^2 u'\rp' = 0 
    \end{equation}
    on $(-r, r)$.
    Assume that $\rho(x) \le \delta$ for $x \in [-r, r]$, then the function 
    \[
        v(x) = \max\{ u(r), u(-r)\} e^{\sqrt{\frac{\tau}{\delta}}(x^2 - r^2)}
    \]
    satisfies $v \ge u$ for all $x \in [-r, r]$, as long as $\tau > 4 \delta$.
\end{lemma}
\begin{proof}
    Note that the function $w(x) := e^{a x^2}$ satisfies
    \begin{align*}
        \tau w - \frac{1}{\rho}\lp \rho^2 w'\rp' &= \rho \lp \tau \rho^{-1} - 4 a x \lp \log \rho\rp' - 2 a \lp 1 + 2 a x^2\rp \rp w(x) \\
        &\ge \rho \lp \tau \delta^{-1} - 4 a r\lp a r +  \max_{x}|\lp\log \rho \rp'| \rp-  2 a \rp w(x)
    \end{align*}
    for $|x| \le r$ with the assumption $\rho(x) \le \delta$. By setting $a^2 =  \tau \delta^{-1}$, then we can bound 
    \begin{align*}
        \tau \delta^{-1} - 4 a r\lp a r +  \max_{x}|\lp\log \rho \rp'| \rp-  2 a  &= a \lp a - 4  r\lp a r +  \max_{x}|\lp\log \rho \rp'| \rp-  2 \rp > 0
    \end{align*}
    as long as 
    \[
        4a r^2 + 4 \max_{x} |\lp\log \rho\rp'| r + (2 - a) < 0,
    \]
    which is ensured for $a > 2$ and $r$ sufficiently small. The condition $a > 2$ is exactly the assumed condition on $\tau > 4 \delta$. 
    We conclude that
    \[
        \tau w - \frac{1}{\rho}\lp \rho^2 w'\rp' > 0, 
    \]
    from which we can see that $v(x) = u(r) e^{\sqrt{\frac{\tau}{\delta}}(x^2 - r^2)}$ as in the lemma statement is a supersolution to the differential equation and satisfies $v \ge u$. 
\end{proof}

Appealing to Lemma \ref{lemma:supersolution-similar-problem} and using the monotonic decreasing nature of the solution $v_0$ on the interval $(0, \frac{\R_s}{2})$ as detailed in Proposition \ref{prop:sim-label-problem}, we can reason that for each of the low-density subintervals $I_i$, the function
\begin{align*}
    v_i(x) &:= \lp \max_{x \in \partial I_i}\ v_0(x) \rp \exp \lp -\sqrt{\frac{\tau}{\delta}} \lp x- x_i\rp^2 - \frac{\ell_i^2}{4} \rp \\
    &\le \begin{cases}
        v_{i-2}(x_{i-2})\exp \lp -\sqrt{\frac{\tau}{\delta}} \lp \lp x- x_i\rp^2 - \frac{\ell_i^2}{4} \rp\rp & \text{ if } i \ge 2\\
        \exp \lp -\sqrt{\frac{\tau}{\delta}} \lp\lp x- x_i\rp^2 - \frac{\ell_i^2}{4} \rp\rp & \text{ otherwise}
    \end{cases}
\end{align*}
serves as an upper bound for the solution $v_0$ on each low-density interval $I_i$. Without loss of generality, assume that $i=n-1$ is the largest index of the low-density subintervals. We can then bound the acquisition function at the midpoint by unraveling the recursion from above as
\begin{align*}
    \mcl A_s\lp \frac{\R_s}{2} \rp &= v_0\lp \frac{\R_s}{2}\rp \le v_{n-1}(x_{n-1}) \le \prod_{i \in \mcl I_\ell} \exp \lp -\sqrt{\frac{\tau}{\delta}}\lp \lp x- x_i\rp^2 - \frac{\ell_i^2}{4} \rp\rp \Big|_{x = x_i} \\
    &= \exp \lp -\sqrt{\frac{\tau}{\delta}}\frac{1}{4}\sum_{i \in \mcl I_\ell} \ \ell_i^2 \rp \le \exp \lp -\sqrt{\frac{\tau}{\delta}}\frac{1}{4n_\ell} |I_{\le \delta}|^2\rp \\
    &= \exp \lp -\sqrt{\frac{\tau}{\delta}}\frac{1}{16n_\ell} \alpha^2(\delta) \R_s^2\rp,
\end{align*}
where we have used Jensen's inequality in the penultimate line.

\subsubsection{Comparing the bounds} \label{smsec:compare-1d-bounds}

Summarizing the bounds of the previous section, the acquisition function values for the different regions will follow the bounds
\begin{align}
    \min_{x \in (0, \R_o)} \ \mcl A_o\lp x\rp &\ge \frac{1}{\sqrt{2}}\exp\lp -\frac{\tau\R_o^2}{8(1 -\lambda)\rho_{o,min}} \rp  &\text{(Opp. Problem)} \\
    \mcl A_s\lp\frac{\R_s}{2}\rp &<  \exp \lp -\sqrt{\frac{\tau}{\delta}}\frac{\alpha^2(\delta) \R_s^2}{16n_\ell} \rp  &\text{(Sim. Problem)}. 
\end{align}

Similar to the setup of Section \ref{subsec:warmup-const-density}, we now analyze the situation when the interval length $\R_o$ is relatively smaller than $\R_s$ to verify the effect of $\tau > 0$ on querying either exploratively (i.e., $\mcl A_s(\R_s/2) < \min \mcl A_o(x)$) or exploitatively (i.e,\ $\mcl A_s(\R_s/2) > \min \mcl A_o(x)$) for $\R_o$ small compared to $\R_s$. We state the following proposition that demonstrates a set of simple conditions under which we are ensured to select an ``explorative'' point prior to an ``exploitative'' point.

\begin{proposition} \label{smprop:1d-result}
    Suppose that the density $\rho(x) \in [\rho_{o,min}, \rho_{o,max}]$ in the oppositely labeled problem region satisfies Assumption \ref{assumption:end-intervals}.
    Suppose also that the density $\rho(x)$ in the similarly labeled problem region satisfies Assumption \ref{assumption:symmetry} and that there exists a value of $\delta \in (\rho_{s,min}, \rho_{s,max})$ such that $\alpha(\delta) = \frac{3}{4}, \frac{1}{2}\rho_{o,min} \le 16\delta \leq \rho_{o,min},$ and $n_\ell = 1$. Further assume that the value of $\epsilon > 0$ from Assumption \ref{assumption:end-intervals} for the oppositely labeled problem interval satisfies
    \begin{equation}\label{eq:eps-condition}
        \tanh(\epsilon) \ge 4 \lp \max_{x\in (\epsilon, \R_o - \epsilon)} \ |\lp \log \rho \rp'(x)| \rp.
    \end{equation}
    Finally, let $\R_o = \beta \R_s$ for some $\beta \in (0, \frac{1}{2\sqrt{2}})$ and the length of the interval $\R_s$ satisfies
    \[
        \R_s^2  \ge 4 \ln (2), 
    \]
    Then we are ensured that 
    \[
        \mcl A_s\lp \frac{\R_s}{2} \rp < \min_{x \in (0, \R_o)} \ \mcl A_o(x)
    \]
    as long as $\tau \in \lp 16\delta, \frac{\rho^2_{o,min}}{16\delta}\rp$.
    
\end{proposition}

\begin{proof}
In order to obtain the desired inequality of $\mcl A_s(\frac{\R_s}{2}) \le \min_x \mcl A_o(x)$, it is equivalent to consider the inequality
\begin{align*}
    \frac{\tau \beta^2 \R_s^2}{8(1-\lambda)\rho_{o,min}} + \frac{\ln 2}{2} \le \sqrt{\frac{\tau}{\delta}}\frac{\alpha^2(\delta) \R_s^2}{16n_\ell}, 
\end{align*}
where the parameter $\lambda$ comes from the requirement of \eqref{eq:gamma-tanh-bound}.
We can rewrite the above as
\begin{equation} \label{eq:tau-condition-messy}
    \R_s^2 \lp \frac{\alpha^2(\delta) }{n_\ell} \sqrt{\frac{\tau}{\delta}} - \frac{2\beta^2 }{(1-\lambda)} \frac{\tau}{\rho_{o,min}}  \rp \ge  8\ln 2.
\end{equation}
Using the assumption that $\tau <\frac{\rho_{o,min}^2}{16\delta}$, then $\sqrt{\frac{\tau}{\delta}} > 4\frac{\tau}{\rho_{o,min}}$ and we can simplify the requirement to satisfy \eqref{eq:tau-condition-messy} to the requirement
\begin{equation}
    \R_s^2 \sqrt{\frac{\tau}{\delta}} \lp \frac{\alpha^2(\delta)}{n_\ell} - \frac{\beta^2}{2(1 - \lambda)} \rp \ge  8\ln 2.
\end{equation}
Now, the condition \eqref{eq:eps-condition} is equivalent to $\theta = \sqrt{2 \tau \rho_{o,min}^{-1}} > 1$ and $\lambda = \frac{1}{2}$, using the assumption $\delta \ge \frac{\rho_{o,min}}{32}$. Thus, plugging in $\lambda = \frac{1}{2}, \alpha(\delta) = \frac{3}{4},n_\ell = 1, \beta < \frac{1}{4}$ and $\tau > 16 \delta$ we have 
\begin{equation}
    \R_s^2 \lp \frac{9}{16}  - \frac{1}{16}  \rp = \frac{1}{2} \R_s^2\ge   2\ln 2,
\end{equation}
from which we can use $\R_s \ge 4 \ln 2$ to conclude that
\[
    \mcl A_s\lp \frac{\R_s}{2} \rp < \min_{x \in \R_o} \ \mcl A_o(x).
\]
\end{proof}

We comment that \eqref{eq:tau-condition-messy} is an inequality that clearly elucidates the effect of each of the parameters $\lambda, \beta, \delta, \rho_{o,min}, \R_s,$ and $\tau$. Namely, as the relative size $\beta$ of the oppositely labeled region ($\R_o = \beta \R_s$) decreases, then the inequality in \eqref{eq:tau-condition-messy} is more easily satisfied. 

\subsection{Proofs from Section \ref{sec:exploration}}
\label{sec:comp-lemmas}

In this section we give the proofs of Theorems \ref{thm:explore_new_cluster} and \ref{thm:explore_labels} from Section \ref{sec:exploration}. We first proceed with some preliminary comparison lemmas for the equation
\begin{equation}\label{eq:rwll}
\tau u - \rho^{-1}\div\left(\gamma\rho^2  \nabla u \right) = 0.
\end{equation}
These local estimates on solutions of \eqref{eq:rwll} will later be used to analyze active learning in the exploration phase. We denote by $B_r(x)$ the open ball of radius $r$ centered at $x$ and write $B_r=B_r(0)$. 
\begin{lemma}[Upper bound]\label{lem:upper_bound}
Let $\tau>0$, $\delta>0$, $r>0$ and $x_0\in \Omega$. Suppose that $\rho \leq \delta$ on $B_r(x_0)\subset \Omega$, let $u\leq 1$ be a subsolution of \eqref{eq:rwll} on $B_r(x_0)$ and assume that
\begin{equation}\label{eq:upper_cond}
3\left(\tfrac{d}{r} + 2\|\nabla \log \rho\|_{L^\infty(B_r(x_0))}\right)\|\gamma\|_{L^\infty(B_r(x_0))} + 3\|\nabla \gamma\|_{L^\infty(B_r(x_0))} \leq  \sqrt{\frac{\tau}{\delta}}.
\end{equation}
Then it holds that
\begin{equation}\label{eq:upper_bound}
\sup_{B_{\frac{r}{2}}(x_0)}u \leq \exp\left(-\frac{r}{4}\sqrt{\frac{\tau}{\delta}}\right).
\end{equation}
\end{lemma}
\begin{proof}
Without loss of generality we may take $x_0=0$.  Define $v(x) = e^{\beta |x|^2}$ for $\beta>0$ to be determined.  Then we have
\[\nabla v(x) = 2x \beta e^{\beta|x|^2},\]
and
\[\nabla^2 v(x) = 2\beta e^{\beta|x|^2}I + 4\beta^2 e^{\beta|x|^2}x\otimes x.\]
Therefore
\[\Delta v(x) = 2\beta e^{\beta|x|^2}\left( d + 2\beta |x|^2\right).\]
We now compute
\begin{align*}
\rho^{-1}\div(\gamma\rho^2 \nabla v) &= \gamma\rho \Delta v + \rho^{-1}\nabla (\gamma\rho^2)\cdot \nabla v\\
&= \rho\left(\gamma \Delta v + (2\gamma\nabla \log \rho + \nabla \gamma)\cdot \nabla v\right)\\
&= 2\beta\rho e^{\beta |x|^2}\left(\gamma d + 2\beta|x|^2 + (2\gamma\nabla \log \rho + \nabla \gamma)\cdot x\right).
\end{align*}
Therefore
\[\tau v - \rho^{-1}\div(\gamma\rho^2 \nabla v) \geq e^{\beta |x|^2}\left( \tau - 
2\beta\delta\left(\gamma d + Kr + 2\beta r^2 \right)\right).\]
where
\[K= \|b\gamma\nabla \log \rho + \nabla \gamma\|_{L^\infty(B_r)}.\]
Now, we assume that $\beta$ is chosen small enough so that 
\begin{equation}\label{eq:beta_bound}
2\beta\delta\left(\gamma d + Kr + 2\beta r^2 \right)\leq \tau.
\end{equation}
This implies that  $v$ is a supersolution of \eqref{eq:rwll}. Likewise, $Av$ is a supersolution for any constant $A>0$. Setting $A=e^{-\beta r^2}$ so that $Av =1 \geq u$ on $\partial B_r$, we can use the comparison principle to obtain that 
\[u(x) \leq Av(x) =  e^{-\beta(r^2 - |x|^2)}\]
for all $x\in B_r$. For $x\in B_{\frac{r}{2}}$ we have $|x|^2 \leq \frac{r^2}{4}$, which yields 
\[\sup_{B_{\frac{r}{2}}}u \leq e^{-\frac{3}{4}\beta r^2}.\]

To complete the proof, we return to the bound \eqref{eq:beta_bound}. This inequality holds provided
\[\gamma d + Kr \leq \beta r^2 \ \ \text{and} \ \ 6\beta^2 r^2 \delta\leq \tau.\]
We thus make the choice
\[\beta =  \frac{\sqrt{\frac{\tau}{\delta}}}{3r},\]
which yields the bound
\[\sup_{B_{\frac{r}{2}}}u \leq \exp\left(-\frac{r}{4}\sqrt{\frac{\tau}{\delta}}\right),\]
provided that
\begin{equation}\label{eq:beta_cond}
\beta r^2 = \frac{r}{3}\sqrt{\frac{\tau}{\delta}}\geq \gamma d + Kr.
\end{equation}
The proof is completed by noting that
\[\gamma d + Kr \leq \left(d + b\|\nabla \log \rho\|_{L^\infty(B_r)}r\right)\|\gamma\|_{L^\infty(B_r)} + \|\nabla \gamma\|_{L^\infty(B_r)}r,\]
and so \eqref{eq:beta_cond} is satisfied when \eqref{eq:upper_cond} holds.
\end{proof}

We now prove a useful lower bound.

\begin{lemma}[Lower bound]\label{lem:lower_bound}
Let $\tau \geq 0$ and $\alpha > d-2$. Let $\L \subset \Omega$ and assume $\L=\L_0\cup \L_1$ where $\L_0$ and $\L_1$ are disjoint, finite, and nonempty. Let $\gamma$ be given by \eqref{eq:gamma}. Suppose that $u$ is a nonnegative solution of \eqref{eq:rwll} on $\Omega\setminus \L$ satisfying $u=0$ on $\L_0$ and $u=1$ on $\L_1$. Let $r>0$ be small enough so that $r \leq \frac{1}{2}\dist(\L_0,\L_1)$
\begin{equation}\label{eq:rtau}
\tau r^d \leq \tfrac{1}{9}(\alpha+2-d)^2\inf_{\L_1 + B_{r}}\rho,
\end{equation}
and
\begin{equation}\label{eq:rcond}
2\|\nabla \log \rho\|_{L^\infty(\L_1+B_r)}(1 + r^{\alpha})r + \alpha r^{\alpha} \leq  \tfrac{1}{4}(\alpha + 2-d).
\end{equation}
Then we have
\begin{equation}\label{eq:lower_bound}
\inf_{\L_1 + B_{\frac{r}{2}}}u \geq 1 - 2^{-\frac{1}{2}(\alpha + 2-d)}.
\end{equation}
\end{lemma}
\begin{proof}
The proof is split into three steps.

1. Define
\[\beta = \frac{1}{2}(\alpha + 2-d)>0. \]
For each $z\in \L_1$ define the functions
\[\phi_z(x) = 1 - \left( \frac{|x-z|}{r}\right)^\beta,\]
and
\[\gamma_z(x) = 1 + \dist(x,\L_0\cup \{z\})^{-\alpha}.\]
We claim that $\phi_z$ satisfies
\[\tau \phi_z - \rho^{-1}\div\left(\gamma_z\rho^2  \nabla \phi_z \right) < 0 \ \ \text{on} \ \ B_r(z).\]
To see this, without loss of generality we take $z=0$, and we note that $\gamma_z(x) = 1 + |x|^{-\alpha}$ on $B_r$, as $r \leq \frac{1}{2}\dist(\L_0,\L_1)$.   We now set $w(x) =|x|^\beta$. Since $\nabla w(x) = \beta|x|^{\beta-2}x$ we have for $x\in  B_r$ that
\begin{align*}
\div(\gamma_z  \rho^2 \nabla w) &=\div(\rho^2(1+ |x|^{-\alpha})\nabla w) \\
&=\beta\div(\rho^2|x|^{\beta-2}x) + \beta\div(\rho^2|x|^{\beta-\alpha-2}x)\\
&=2\beta|x|^{\beta-\alpha-2}(1 + |x|^{\alpha})\rho \nabla \rho \cdot x +\beta\rho^2\div(|x|^{\beta-2}x) + \beta\rho^2\div(|x|^{\beta-\alpha-2}x)\\
&=\beta\rho^2|x|^{\beta-\alpha-2}\left[ 2(1 + |x|^{\alpha})\nabla \log\rho \cdot x + (d+\beta-2)|x|^{\alpha} + d + \beta-\alpha-2  \right]\\
&\leq\beta\rho^2|x|^{\beta-\alpha-2}\left[ 2\|\nabla \log \rho\|_{L^\infty(B_r)}(1 + |x|^{\alpha})|x| + \alpha|x|^{\alpha} - \tfrac{1}{2}(\alpha + 2-d) \right].
\end{align*}
Invoking \eqref{eq:rcond} we have 
\[\div(\gamma_z \rho^2\nabla w) \leq-\frac{1}{8}(\alpha+2-d)^2\rho^2|x|^{\beta-\alpha-2}\]
for $x\in B_r$. It follows that
\[\tau \phi_z-\rho^{-1}\div(\gamma_z \rho^2\nabla \phi_z) \leq \tau-\frac{1}{8}(\alpha+2-d)^2r^{\beta}\rho |x-z|^{\beta-\alpha-2} \ \ \text{on} \ \ B_r(z).\]
Since $|x-z|\leq r$ on $B_r(z)$ and $2\beta - \alpha-2 = -d$ we have
\[\tau \phi_z-\rho^{-1}\div(\gamma_z \rho^2\nabla \phi_z) \leq \tau-\frac{1}{8}(\alpha+2-d)^2\rho r^{-d} < 0 \ \ \text{on} \ \ B_r(z),\]
by assumption \eqref{eq:rtau}. This establishes the claim.

2. We now define 
\[\phi(x) = \max_{z\in \L_1}\phi_z(x),\]
and we claim that $\phi \leq u$ on $\L_1 + B_r$. This follows from a careful application of the comparison theorem, which we give for completeness. Let $x_0\in \L_1+\bar{B_r}$ be a point where $\phi - u$ attains its maximum value. If $x_0 \in \partial (\L_1 + B_r)$, where $\phi=0\leq u$, then we are done, so we may assume $x_0\in \L_1+B_r$. Likewise, if $x_0=z$ for some $z\in \L_1$ then the proof is also complete, since $\phi(z)=u(z)=1$, so we may assume $x_0\not\in \L_1$. Now let 
\[Z = \{z\in \L_1 \, : \, \phi_z(x_0)=\phi(x_0)\}.\]
For any $z\in Z$, we have that $\phi_z - u$ has a local maximum at $x_0$. By elliptic regularity, $u$ is $C^{1,\delta}$ for an $0 < \delta<1$ and so we have $\nabla u(x_0)=\nabla \phi_z(x_0)$. It follows that $Z=\{z\}$ is a singleton set and $\gamma=\gamma_z$ in a neighborhood of $x_0$. Thus, $u$ satisfies 
\[\tau u(x_0)-\rho^{-1}\div(\gamma_z \rho^2\nabla u)\big\vert_{x_0} = 0.\]
Therefore, the difference $v:=\phi_z - u$ satisfies 
\[\tau v(x_0)-\rho^{-1}\div(\gamma_z \rho^2\nabla v)\big\vert_{x_0} < 0.\]
When $\tau =0$ this is a contradiction to the fact that $x_0$ is a local max of $v$. When $\tau > 0$ the local maximum property implies
\[\tau v(x_0) < \rho^{-1}\div(\gamma_z \rho^2\nabla v)\big\vert_{x_0} \leq 0,\]
and so $v(x_0)\leq 0$. This establishes the claim that that $\phi \leq u$.

3. We now simply note that on $\L_1 + B_{\frac{r}{2}}$ we have
\[u \geq \inf_{\L_1 + B_{\frac{r}{2}}} \phi \geq 1 - 2^{-\frac{1}{2}(\alpha + 2-d)}, \]
which completes the proof.
\end{proof}

We now give the proofs of the our main results. 

\begin{proof}[Proof of Theorem \ref{thm:explore_new_cluster}]
The proof is a direct application of Lemma \ref{lem:upper_bound}  with $r=s$. For any point $x_0\in \partial (\C+B_s)$ we have $B_s(x_0)\subset \partial_{2s}\C$ and by Lemma \ref{lem:upper_bound} we have that 
\[\sup_{B_{\frac{s}{2}}(x_0)}u_i \leq \exp\left(-\frac{s}{4}\sqrt{\frac{\tau}{\delta}}\right),\]
for all $i\in \{1,\dots,C\}$ provided that \eqref{eq:upper_cond} holds. By the assumption that  $\L\cap (\C+B_{R+2s})=\varnothing$ we see that 
\[\|\gamma\|_{L^\infty(B_s(x_0))} \leq 1 + R^{-\alpha} \ \ \text{and} \ \ \|\nabla \gamma\|_{L^\infty(B_s(x_0))}\leq R^{-\alpha-1},\]
and so \eqref{eq:explore_cond} implies \eqref{eq:upper_cond}. Therefore, each $u_i$ satisfies 
\[0 \leq u_i \leq \exp\left(-\frac{s}{4}\sqrt{\frac{\tau}{\delta}}\right) \ \ \text{on} \ \ \partial(\C + B_s).\]
Since $\L\cap \partial_s \C = \varnothing$, the maximum priciple yields the same estimate on $\C+B_s$, which completes the proof.
\end{proof}

\begin{proof}
The proof is a direct application of Lemma \ref{lem:lower_bound}, except that we replace $r$ from Lemma \ref{lem:lower_bound} with $2r$ in this result, to simplify the final statement \eqref{eq:acq_lower}. Indeed, Lemma \ref{lem:lower_bound} applies to each $u_i$, yielding
\[\inf_{\L_i + B_r} u_i \geq 1 - 2^{-\frac{1}{2}(\alpha + 2-d)}.\]
Since $\A \geq u_i$, we have
\[\inf_{\L_i + B_r} \A \geq 1 - 2^{-\frac{1}{2}(\alpha + 2-d)},\]
for all $i\in \{1,\dots,C\}$, which completes the proof.
\end{proof}


\section{Alternatives to acquisition function maximization}

Up to this point, we have framed the active learning process in terms of maximizing an acquisition function on a discrete set of currently unlabeled points. Similar to a reinforcement learning \textit{policy}, we refer to this as the active learning policy: 
\begin{policy}
    Given an acquisition function $\mcl A$, select to query the maximizer 
    \[
    x^\ast = \argmax_{x \in \mcl U} \mcl A(x)
    \]
    at each iteration.
\end{policy}
We briefly discuss some alternatives to this straightforward maximizer policy that will be advantageous for the theoretical results of Section \ref{sec:theory}. In our experiments on the larger datasets, however, we do not see a significant performance improvement over the straightforward selection of the maximizer. As such, we only display the results for the maximizers in Sections \ref{sec:larger-datasets} and \ref{smsec:imbalanced-results}. See Figure \ref{fig:compare-kde-prop-results}, which we include for completeness to demonstrate how these alternatives give nearly the same performance on these datasets from Section \ref{sec:larger-datasets}.

\subsubsection{Kernel density estimator}
\label{sec:kde}

Given that the goal of active learning queries is to select unlabeled points that are most informative for the underlying classifier, it is reasonable to suggest that outliers of the dataset might not constitute the most useful points to be considered for querying. A simple modification to the active learning policy is to exclude outlier from  the set of unlabeled points considered for querying. Given a kernel density estimator (KDE) value for each point in the dataset $\hat{\rho}: \mcl X \rightarrow \mbb R_+$, let $\mcl U_{\hat{\rho}} = \{x \in \mcl U : \hat{\rho}(x) > T_{\hat{\rho}}\}$, where $T_{\hat{\rho}} > 0$ is a specified threshold. Then, we propose a \textit{KDE policy}:
\begin{policy}
    Given an acquisition function $\mcl A$, select to query the maximizer \[
        x^\ast = \argmax_{x \in \mcl U_{\hat{\rho}}} \mcl A(x)
    \] at each iteration.
\end{policy}
For our tests, we use the k-nearest neighbor KDE:
\[
    \hat{\rho}_{knn}(x) \propto \frac{1}{\|x - \tilde{x}_{k}\|_2},
\]
where $\tilde{x}_k$ is the $k^{th}$ nearest neighbor to $x$, and we set the threshold $T_{\hat{\rho}}$ to be the $10^{th}$ percentile:
\[
    T_{\hat{\rho}} \equiv \inf\ \left\{ t \in \mbb R: \sum_{x \in \mcl X} \chi\lp \hat{\rho}_{knn}(x) < t\rp \le 0.1 |\mcl X|\right\}
\]

\subsubsection{Proportional sampling}
Another alternate policy that we propose is to sample query points \textit{proportional to acquisition function values}, which we may state as:
\begin{policy}
    Given an acquisition function $\mcl A$, select to query the maximizer $x^\ast \in \mcl U$ with probability
    \[
        p(x) \propto \exp(\mcl A(x) / T), \quad T > 0
    \]
    at each iteration.
\end{policy}
Notice that as $T \rightarrow 0^+$, we recover a distribution that is only supported on the set of maximizers of $\mcl A$; that is, $p(x) > 0$ for $x \in \argmax_{\mcl U} \ \mcl A$. Similarly, as $T \rightarrow \infty$, we recover uniform sampling across the unlabeled set $\mcl U$; that is, no acquisition function information is included. 

Since we want to emphasize the upper end of acquisition function values, we derive a method for choosing $T$ that scales with acquisition function values. Let $N_k$ be the size of cluster $C_k$, then let $\hat{K}$ be defined as
\[
    \hat{K} = \frac{N}{\min_{k=1, \ldots, K} N_k}.
\]
Then, if we a priori knew the number of clusters $K$ in a dataset, we could set this $\hat{K}$ value. If each cluster is of equal size, then $K = \hat{K}$. Furthermore, let $M_{\mcl A} = \max \mcl A(x)$ and define 
\[
    \Phi\lp \mcl A; P\rp = \inf_{t \in \mbb R} \ \left\{  \frac{1}{|\mcl U|}\sum_{x \in \mcl U} \mathbbm{1}\lp \mcl A(x) \le t \rp  \ge \frac{P}{100}\right\}
\]
to be the $P^{th}$ percentile of the set of acquisition function values $\{\mcl A(x)\}_{x \in \mcl U}$. Then we can define the value
\[
    T_0 \equiv  \frac{ M_{\mcl A} - \Phi\lp\mcl A; 100 (  1 - \frac{1}{\hat{K}}) \rp}{M_{\mcl A}}
\]
as the relative difference between the maximizer's and the $100\lp \frac{\hat{K}-1}{\hat{K}}\rp$-percentile's acquisition function values. As this gets smaller, we more dramatically emphasize the top $\frac{1}{\hat{K}}$ fraction of points with highest acquisition function values. 
We then set $T$ as
\[
     T \equiv \max( \epsilon_o, T_0),
\]
where $\epsilon_{o} > 0$ is a constant designed to prevent numerical overflow in the computation of the probabilities $p(x) \propto \exp(\mcl A(x)/T)$. Given a maximum floating point value of $\mcl E_{max}$, we can ensure no overflow occurs if
\[
    \sum_{z \in \mcl U} e^{\frac{\mcl A(z)}{T}} < \frac{\mcl E_{max}}{|\mcl U|}.
\]  
This condition allows us to safely set the condition 
\[
    T > \frac{M_{\mcl A}}{\ln(\mcl E_{max}) - \ln(|\mcl U|)} \equiv \epsilon_o.
\]


\begin{figure}
    \centering
    \subfigure[MNIST Acc.]{\includegraphics[width=0.32\textwidth]{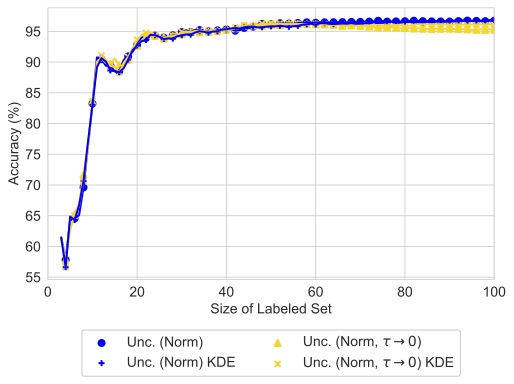}}
    \subfigure[FASHIONMNIST Acc.]{\includegraphics[width=0.32\textwidth]{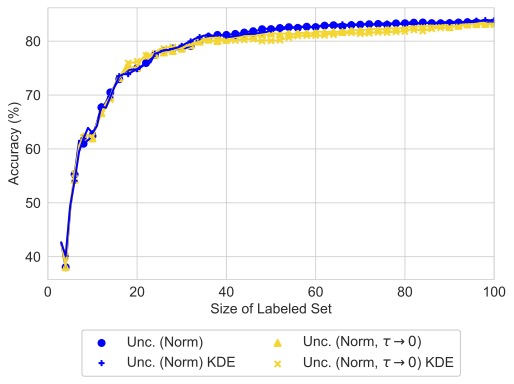}}
    \subfigure[EMNIST Acc.]{\includegraphics[width=0.32\textwidth]{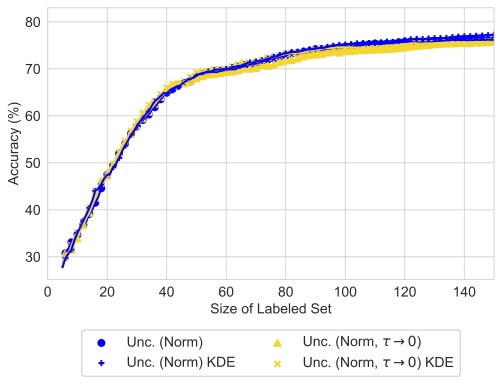}} 
    
    \subfigure[MNIST Acc.]{\includegraphics[width=0.32\textwidth]{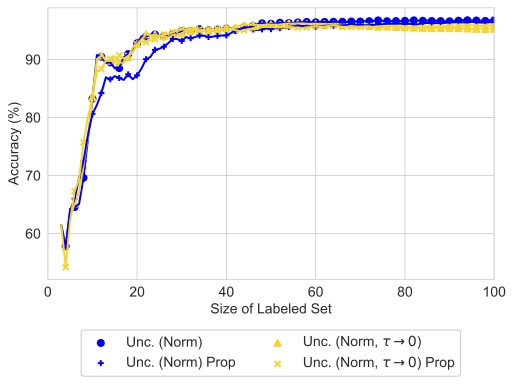}}
    \subfigure[FASHIONMNIST Acc.]{\includegraphics[width=0.32\textwidth]{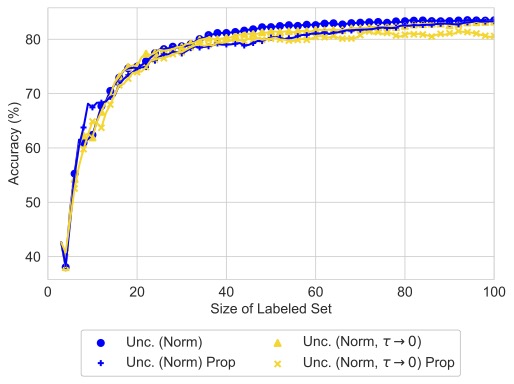}}
    \subfigure[EMNIST Acc.]{\includegraphics[width=0.32\textwidth]{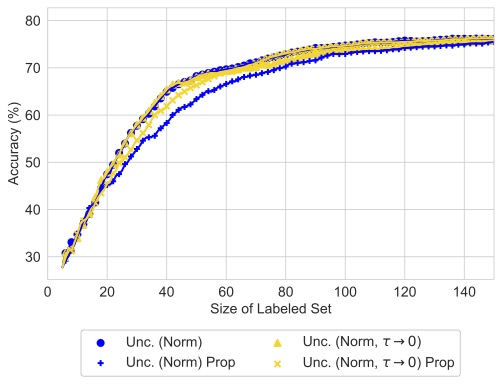}}
    \caption{Comparison of alternative policies to selecting straightforward maximizers (Policy 1). Accuracies of each acquisition function is reported in the PWLL model. Blue-colored curves use the Unc.~(Norm) acquisition function, while the yellow-colored use Unc.~(Norm, $\tau \rightarrow 0$). There is no significant improvement of using the alternatives over the original policy, so we only plot the original policy's results in the main comparison of results in Section \ref{sec:larger-datasets}.}
    \label{fig:compare-kde-prop-results}
\end{figure}

\end{document}